\documentclass{article}

\PassOptionsToPackage{}{natbib}

\usepackage[preprint]{main}
\usepackage[utf8]{inputenc}
\usepackage[T1]{fontenc}
\usepackage{hyperref}
\usepackage{url}
\usepackage{booktabs}
\usepackage{amsfonts}
\usepackage{nicefrac}
\usepackage{microtype}
\usepackage{xcolor}
\usepackage{graphicx}
\usepackage{amsmath}
\usepackage{amssymb}
\usepackage{amsthm}
\usepackage{cleveref}
\usepackage{empheq}
\usepackage{algorithmic}
\usepackage{algorithm}
\usepackage{authblk}
\usepackage{enumitem}

\theoremstyle{plain}
\newtheorem{theorem}{Theorem}[section]
\newtheorem{proposition}[theorem]{Proposition}
\newtheorem{lemma}[theorem]{Lemma}

\theoremstyle{definition}
\newtheorem{definition}[theorem]{Definition}
\newtheorem{assumption}[theorem]{Assumption}
\theoremstyle{remark}
\newtheorem{remark}[theorem]{Remark}

\crefname{definition}{Definition}{Definitions}
\crefname{theorem}{Theorem}{Theorems}
\crefname{assumption}{Assumption}{Assumptions}
\crefname{lemma}{Lemma}{Lemmas}
\crefname{proposition}{Proposition}{Propositions}
\crefname{remark}{Remark}{Remarks}

\crefname{section}{Section}{Sections}
\crefname{equation}{Eq.}{Eqs.}

\newcommand{\indep}{\mathop{\perp\!\!\!\perp}}

\newcommand{\dd}{\mathop{}\!\mathrm{d}}

\title{DOLCE: Decomposing Off-Policy Evaluation/Learning\\into Lagged and Current Effects}

\author[1,2,$\ast$]{Shu Tamano}
\affil[1]{Department of Multidisciplinary Sciences, Graduate School of Arts and Sciences, The University of Tokyo, 3-8-1 Komaba, Meguro-ku, Tokyo 153-8902, Japan}
\affil[2]{Department of Epidemiology, National Institute of Infectious Diseases, Japan Institute for Health Security, 1-23-1 Toyama, Shinjuku-ku, Tokyo 162-0052, Japan}
\affil[$\ast$]{Email: tamano-shu212@g.ecc.u-tokyo.ac.jp}

\begin{document}

\maketitle

\begin{abstract}
    Off-policy evaluation and learning in contextual bandits use logged interaction data to estimate and optimize the value of a target policy.
    Most existing methods require sufficient action overlap between the logging and target policies, and violations can bias value and policy gradient estimates.
    To address this issue, we propose \textbf{DOLCE} (\textbf{D}ecomposing \textbf{O}ff-policy evaluation/learning into \textbf{L}agged and \textbf{C}urrent \textbf{E}ffects), which uses only lagged contexts already stored in bandit logs to construct lag-marginalized importance weights and to decompose the objective into a support-robust lagged correction term and a current, model-based term, yielding bias cancellation when the reward-model residual is conditionally mean-zero given the lagged context and action.
    With multiple candidate lags, DOLCE softly aggregates lag-specific estimates, and we introduce a moment-based training procedure that promotes the desired invariance using only logged lag-augmented data.
    We show that DOLCE is unbiased in an idealized setting and yields consistent and asymptotically normal estimates with cross-fitting under standard conditions.
    Our experiments demonstrate that DOLCE achieves substantial improvements in both off-policy evaluation and learning, particularly as the proportion of individuals who violate support increases.
\end{abstract}

\section{Introduction}
\label{sec:introduction}
Many real-world decision-making systems, including recommendation engines \citep{li2010contextual, swaminathan2017off, saito2021counterfactual}, ad-placement platforms \citep{bottou2013counterfactual}, and clinical decision support \citep{qian2011performance, liao2021off}, interact with their users through the contextual bandit framework.
A logging (behavior) policy repeatedly observes a context, selects an action, and records the realized reward, yielding logged bandit feedback.
A central goal is to assess a target policy using logged bandit feedback, thereby reducing the cost and risk of online deployment and experimentation.
This is known as \emph{off-policy evaluation} (OPE) \citep{dudik2011doubly,li2011unbiased,wang2017optimal}.
When the target policy is unknown and must be optimized from logs, the corresponding problem is \emph{off-policy learning} (OPL), also referred to as learning from logged bandit feedback or counterfactual risk minimization \citep{swaminathan2015batch, swaminathan2015counterfactual, joachims2018deep}.

Most practical OPE and OPL methods \citep{dudik2011doubly, swaminathan2015batch, wang2017optimal, farajtabar2018more, su2019cab, su2020doubly, metelli2021subgaussian} rely on an overlap (support) condition:
for any context, every action assigned positive probability by the target policy must also be taken with positive probability by the logging policy.
This condition is fundamental for identifiability.
When it fails, the reward distribution for unsupported target actions is not identified from the logs, importance-weighted estimators become undefined or rely on extrapolation, and both evaluation and optimization can be biased and unstable \citep{bottou2013counterfactual,sachdeva2020off,felicioni2022off}.
Such overlap violations are common in real systems because logging policies often implement hard eligibility rules or safety constraints.
For instance, clinical guidelines may preclude specific treatments for particular presentations, and advertising systems may never display certain ads outside predefined criteria.
Moreover, overlap can deteriorate rapidly in high-dimensional settings, making violations and near-violations practically unavoidable in large-scale applications \citep{d2021overlap}.
We defer a broader discussion of related work to Appendix~\ref{app:sec:related-work}.

To overcome this issue, we propose \textbf{DOLCE} (\emph{\textbf{D}ecomposing} \emph{\textbf{O}ff-Policy} \emph{Evaluation/Learning into \textbf{L}agged and \textbf{C}urrent \textbf{E}ffects}).
The key idea is to leverage lagged contexts that are routinely available in logged interactions but are typically unused by standard OPE/OPL pipelines.
Instead of defining importance weights using the action probability at the current context (where support violations arise), DOLCE constructs weights from action probabilities marginalized to a lagged context.
This design only requires overlap at the level of a lagged context, which can hold even when common support fails at the current context.
Under the standard premise that the logging action depends on the current context, the resulting lag-based reweighting admits an exact bias-cancellation mechanism.

We further provide an estimation procedure tailored to the model conditions required by this bias cancellation.
In particular, our analysis shows that unbiasedness can be achieved under \emph{residual invariance}, which requires that the reward-model error does not vary with the current context once the lagged context and action are fixed.
This condition strictly relaxes global reward-model correctness and generalizes the local-correctness intuition studied in \citet{saito2023off,saito2025potec}.
Building on this perspective, we develop practical estimators for both OPE and gradient-based OPL, together with sample splitting and cross-fitting to support valid inference \citep{chernozhukov2018double}.
We establish oracle unbiasedness under lag-level overlap and residual invariance, and we prove consistency and asymptotic normality under standard nuisance-estimation rate conditions.
Empirically, DOLCE substantially reduces bias relative to standard methods as support violations become more severe, and it yields more effective off-policy optimization than conventional baselines.
\section{Preliminaries}
\label{sec:preliminaries}

\subsection{Setup}
\label{subsec:setup}
Let $(\Omega, \mathcal{F}, P)$ be a probability space.
We consider a contextual bandit with context space $\mathcal{X} \subset \mathbb{R}^{d}$, a finite action space $\mathcal{A}$, and reward space $\mathcal{R}\subset \mathbb{R}$.
For each subject $i\in [n] := \{1,\ldots, n\}$, we observe $(X_i,A_i,R_i)$ and collect the dataset $\mathcal{D} := \{(X_{i}, A_{i}, R_{i})\}_{i=1}^{n}$, where $(X_i, A_i, R_i)$ are independent and identically distributed (i.i.d.) draws from the logging distribution induced by an unknown context marginal $p(x)$, a logging policy $\pi_0(a\mid x)$, and an unknown reward conditional distribution $p(r\mid x, a)$.
When lagged contexts are available (Section~\ref{subsec:lagged-formulation}), we augment the observation with lagged covariates and allow the conditional reward law to depend on both current and lagged contexts.

For identifiability of causal effects, we adopt the standard potential outcome notation \citep{neyman1923sur, rubin1974estimating}.
Let $\{R(a): a\in \mathcal{A}\}$ denote the potential rewards given context $X$.
Throughout, we adopt the standard identifying assumptions, namely consistency and unconfoundedness, stated below.

\begin{assumption}[Consistency]
\label{ass:ci1:consistency}
    For each subject, the observed reward equals the potential reward under the action actually taken.
    That is, $R=R(A)$ almost surely under $P$.
\end{assumption}

\begin{assumption}[Unconfoundedness]
\label{ass:ci2:unconfoundedness}
    Conditional on the observed context, the logging action is independent of all potential rewards.
    That is, $\{R(a)\}_{a\in\mathcal{A}} \indep A\mid X$ under $P$.
    Equivalently, the observed data admit the factorization $P(\dd x, \dd a, \dd r) = p(x)\pi_0(a\mid x)p(r\mid x, a)\dd x \dd a \dd r$.
\end{assumption}

\begin{assumption}[Positivity / overlap for contextual bandits]
\label{ass:positivity-cb}
    For each $x$ and each $a$ with $\pi_{\theta}(a\mid x) > 0$, we have $\pi_{0}(a\mid x) > 0$ (common support).
    For gradient-based OPL, we typically require full support: $\pi_{0}(a\mid x) > 0$ for all $(a, x)$, and we restrict to policy classes with $\pi_{\theta}(a\mid x)>0$ when gradients are used.
\end{assumption}

Define the mean reward function $q(x,a) := \mathbb{E}_{P}[R\mid X=x, A=a]$.
Let $\pi_{\theta}(a\mid x)$ denote a target policy parameterized by $\theta\in \Theta \subset \mathbb{R}^m$.
The \emph{policy value} is defined as follows:
\begin{equation*}
    \nonumber
    V(\pi_{\theta})
    :=
    \mathbb{E}_{P}\!\left[
        \sum_{a\in\mathcal{A}}\pi_{\theta}(a\mid X)q(X,a)
    \right]
    =
    \mathbb{E}_{p(x)\pi_{\theta}(a\mid x)p(r\mid x, a)}[R]
    .
\end{equation*}

\paragraph{Off-policy evaluation (OPE).}
Given $\mathcal{D}$ and a fixed target policy $\pi_{\theta}$, the goal is to estimate $V(\pi_{\theta})$.

\paragraph{Off-policy learning (OPL).}
Given $\mathcal{D}$ and a policy class $\{\pi_{\theta}: \theta\in\Theta\}$, the goal is to solve $\theta^{\star}\in\arg\max_{\theta\in\Theta} V(\pi_{\theta})$.
In gradient-based OPL, we additionally assume $\theta \mapsto \pi_{\theta}(a\mid x)$ is differentiable and define the score $s_{\theta}(a\mid x) := \nabla_{\theta}\log \pi_{\theta}(a\mid x)$.

Under Assumptions~\ref{ass:ci1:consistency}--\ref{ass:positivity-cb}, $V(\pi_{\theta})$ and $\nabla_{\theta} V(\pi_{\theta})$ are identifiable from $P$ via observed-data functionals.

\subsection{Baselines}
\label{subsec:baselines}

\paragraph{OPE baselines.}
Let $\hat{q}(x,a)$ be a reward model.
Standard OPE estimators include the Direct Method (DM), Inverse Propensity Score (IPS), and Doubly Robust (DR):
\begin{align}
% \label{eq:dm}
\nonumber
    \hat{V}_{\mathrm{DM}}(\pi_{\theta})
    &:=
    \frac{1}{n}\sum_{i=1}^{n}\sum_{a\in\mathcal{A}}\pi_{\theta}(a\mid X_i)\hat{q}(X_i, a)
    ,
    \\
\label{eq:ips}
    \hat{V}_{\mathrm{IPS}}(\pi_{\theta})
    &:=
    \frac{1}{n}\sum_{i=1}^{n}\frac{\pi_{\theta}(A_i\mid X_i)}{\pi_0(A_i\mid X_i)}R_i
    ,
    \\
\label{eq:dr}
    \hat{V}_{\mathrm{DR}}(\pi_{\theta})
    &:=
    \frac{1}{n}\sum_{i=1}^n
    \Bigl[
        \frac{\pi_{\theta}(A_i\mid X_i)}{\pi_0(A_i\mid X_i)}\{R_i-\hat{q}(X_i, A_i)\}
        +
        \sum_{a\in\mathcal{A}}\pi_{\theta}(a\mid X_i)\hat{q}(X_i, a)
    \Bigr]
    .
\end{align}
With the true propensities, $\hat{V}_{\mathrm{IPS}}$ is unbiased under Assumption~\ref{ass:positivity-cb}.
$\hat{V}_{\mathrm{DM}}$ is unbiased when $\hat{q}=q$.
$\hat{V}_{\mathrm{DR}}$ is doubly robust in the sense of causal inference;
see, e.g., \citet{li2011unbiased,dudik2011doubly}.

\paragraph{OPL baselines.}
OPL is commonly addressed by
(i) regression-based learning (fit $\hat{q}$ and act greedily) or
(ii) gradient-based learning.
For gradient-based OPL, the standard identity (log-derivative trick) yields
\begin{equation*}
% \label{eq:pg-identity}
    \nabla_{\theta} V(\pi_{\theta})
    =
    \mathbb{E}_{P}\Bigl[
        \sum_{a\in\mathcal{A}}\pi_{\theta}(a\mid X)q(X,a)s_{\theta}(a\mid X)
    \Bigr]
    .
\end{equation*}
Corresponding IPS and DR policy-gradient estimators are obtained by importance weighting (and reward-model augmentation); see Appendix~\ref{app:subsec:baseline-gradients} for the explicit formulas.
These estimators are unbiased under Assumption~\ref{ass:positivity-cb} and (for DR-type variants) when $\hat{q}=q$.

\subsection{Support Violation}
\label{subsec:support-violation}

When Assumption~\ref{ass:positivity-cb} fails, IPS/DR (and their gradient analogues) are generally biased because logged data provide no information about rewards for actions that never occur under the logging policy at a given context.
Define the unsupported action set as
\begin{equation*}
    \mathcal{U}(x;\pi_{\theta}, \pi_0)
    :=
    \{a\in \mathcal{A}: \pi_{\theta}(a\mid x) > 0
    ,\;
    \pi_{0}(a\mid x) = 0\}
    .
\end{equation*}
Under Assumptions~\ref{ass:ci1:consistency} and \ref{ass:ci2:unconfoundedness}, the bias of IPS is determined by the target mass assigned to $\mathcal{U}(X;\pi_\theta,\pi_0)$, while the bias of DR is determined by reward-model error on $\mathcal{U}(X;\pi_\theta, \pi_0)$ \citep{sachdeva2020off,felicioni2022off}; see Appendix~\ref{app:subsec:impact-of-support-violation} for details.
\section{Proposed Method}
\label{sec:proposed-method}
We propose \textbf{DOLCE}, a lag-aware doubly robust framework for off-policy evaluation and learning under support violations.
DOLCE replaces the usual current-context importance weights with lag-marginal ratios based on past contexts and achieves exact bias cancellation under a residual-invariance condition (Assumption~\ref{ass:ri}) together with lag overlap (Assumption~\ref{ass:common-lag-support}).
For multiple candidate lags, DOLCE implements a one-good-lag principle via softmin aggregation weighted by an approximate local correctness (ALC) score.

\subsection{Lagged Formulation and One-Good-Lag Principle}
\label{subsec:lagged-formulation}

\paragraph{Augmented data with lagged contexts.}
In many contextual bandit systems, contexts are observed repeatedly over time.
We assume that, in addition to the current context $X$, we have access to a finite set of lag contexts
\begin{equation*}
    \boldsymbol{X}_0
    :=
    (X^{(1)},\ldots,X^{(K)})
    ,
    \quad
    X^{(k)}\in\mathcal{X}_0^{(k)}
    ,
\end{equation*}
corresponding to time lags $\mathcal{L} = \{\ell_1,\ldots, \ell_K\}$ (e.g., $X^{(k)} = X_{t-\ell_k}$).
Accordingly, the lagged dataset becomes
\begin{equation*}
    \mathcal{D}_{\mathrm{lag}}
    :=
    \{(\boldsymbol{X}_{0i}, X_i, A_i, R_i)\}_{i=1}^n
    ,
\end{equation*}
with i.i.d.\ draws $(\boldsymbol{X}_{0}, X, A, R)\sim P$.

For each lag $k\in[K]$, define the conditional mean reward
\begin{equation*}
    q_{k}(x,x^{(k)},a)
    :=
    \mathbb{E}_{P}[R\mid X=x, X^{(k)}=x^{(k)},A=a]
    .
\end{equation*}

\paragraph{Current-action sufficiency.}
We formalize the common contextual-bandit assumption that the logging policy depends only on the current context.

\begin{assumption}[Current-action sufficiency]
\label{ass:current-action-sufficiency}
    For each $k\in[K]$, we have $A\indep X^{(k)}\mid X$ under $P$.
    Equivalently, $\pi_0(a\mid X, X^{(k)}) = \pi_0(a\mid X)$ $P$-a.s.
\end{assumption}

Under Assumption~\ref{ass:current-action-sufficiency}, $q$ and $q_k$ are linked via iterated expectation, yielding an equivalent representation of $V(\pi_\theta)$ in terms of $q_k$; see Appendix~\ref{app:subsec:q-qk-relation}.

\paragraph{Lag-marginalized policies.}
For OPE/OPL, the key objects are the action probabilities marginalized to a lag context:
\begin{align*}
    % \label{eq:barpi-target}
    \bar{\pi}_{\theta,k}(a\mid x^{(k)})
    &:=
    \mathbb{E}_{P}[\pi_{\theta}(a\mid X) \mid X^{(k)} = x^{(k)}]
    ,
    \\
    % \label{eq:barpi_logging}
    \bar{\pi}_{0,k}(a\mid x^{(k)})
    &:=
    P(A=a\mid X^{(k)} = x^{(k)})
    .
\end{align*}
Note that $\bar{\pi}_{0,k}$ is identifiable directly from $P$ without Assumption~\ref{ass:current-action-sufficiency} by modeling $P(A\mid X^{(k)})$.
Assumption~\ref{ass:current-action-sufficiency} is used for the key bias-cancellation argument in DOLCE.

\begin{assumption}[Common lag support / lag overlap]
\label{ass:common-lag-support}
    Fix $k\in[K]$ and define
    \begin{equation*}
        \mathcal{S}_{\theta,k}
        :=
        \{(x^{(k)},a): \bar{\pi}_{\theta,k}(a\mid x^{(k)}) > 0\}
        .
    \end{equation*}
    For $P$-a.e. $x^{(k)}$ and all $a\in\mathcal{A}$,
    \begin{equation*}
        \bar{\pi}_{\theta,k}(a\mid x^{(k)}) > 0
        \Rightarrow
        \bar{\pi}_{0,k}(a\mid x^{(k)}) > 0
        .
    \end{equation*}
    In addition, there exists $\epsilon > 0$ such that
    \begin{equation*}
        \mathbb{E}_{P}\Biggl[
            \sum_{a\in\mathcal{A}}
            \boldsymbol{1}\bigl\{
                (X^{(k)},a)\in \mathcal{S}_{\theta,k}
            \bigr\}
            \bar\pi_{0,k}(a\mid X^{(k)})
        \Biggr]
        \ge
        \epsilon
        .
    \end{equation*}
\end{assumption}

\paragraph{Residual invariance and one-good-lag.}
We introduce a model property that is strictly weaker than global correctness of $q_k$.
Define the residual for a candidate reward model $\tilde{q}_k$ by
\begin{equation*}
    \Delta_k(x,x^{(k)},a;\tilde{q}_k)
    :=
    q_k(x,x^{(k)},a)-\tilde{q}_k(x,x^{(k)},a)
    .
\end{equation*}

\begin{assumption}[Residual invariance for lag $k$]
\label{ass:ri}
    Let $\tilde{q}_k: \mathcal{X}\times \mathcal{X}_0^{(k)}\times \mathcal{A} \to \mathbb{R}$ be measurable.
    There exists a measurable function $\delta_k: \mathcal{X}_0^{(k)}\times \mathcal{A}\to \mathbb{R}$ such that, for every $a\in \mathcal{A}$,
    \begin{equation*}
        \Delta_k(X,X^{(k)},a;\tilde{q}_k)
        =
        \delta_k(X^{(k)},a)
        \quad
        P\text{-a.s.}
    \end{equation*}
\end{assumption}
See Appendix~\ref{app:subsec:residual-invariance} for equivalent characterizations and its connection to local correctness (see, e.g., \citet{saito2023off, saito2025potec}).

Residual invariance generalizes the local-correctness intuition \citep{saito2023off, saito2025potec}:
$\tilde{q}_k$ need not approximate $q_k$ globally as long as its error does not vary with the current context once $(X^{(k)}, A)$ is fixed.

For multi-lag settings, we adopt the following weak requirement.
\begin{assumption}[One-good-lag residual invariance]
\label{ass:one-good-lag}
    There exists at least one lag $k^{\star}\in[K]$ such that the employed reward model $\tilde{q}_{k^{\star}}$ satisfies Assumption~\ref{ass:ri}.
\end{assumption}

\paragraph{Softmin aggregation via ALC.}
To select a reliable lag in a multi-lag setting, we quantify the degree of residual-invariance violation using an approximate local correctness (ALC) score
\begin{equation*}
    \mathrm{ALC}_k(\tilde{q}_k)
    :=
    \mathbb{E}_P\bigl[
        \mathbb{V}[\Delta_k(X,X^{(k)}, A;\tilde{q}_k)\mid X^{(k)}, A]
    \bigr]
    .
\end{equation*}
By construction, $\mathrm{ALC}_{k}(\tilde{q}_k) = 0$ if Assumption~\ref{ass:ri} holds.
We estimate $\mathrm{ALC}_k$ (or a lower bound thereof) and define lag weights via softmin:
\begin{equation}
\label{eq:softmin-weights}
    \alpha_k
    :=
    \frac{\exp\{-\widehat{\mathrm{ALC}}_k/\tau\}}{\sum_{j=1}^K \exp\{-\widehat{\mathrm{ALC}}_j/\tau\}}
    ,
    \quad
    \tau>0
    .
\end{equation}
When $\tau \downarrow 0$, $\alpha_k$ concentrates on the empirically best lag, implementing the one-good-lag principle.

\subsection{DOLCE for OPE}
\label{subsec:dolce-ope}

\paragraph{Lagged importance weights.}
For each lag $k$, define the lag weight
\begin{equation*}
    w_k(x^{(k)},a)
    :=
    \frac{\bar{\pi}_{\theta,k}(a\mid x^{(k)})}{\bar{\pi}_{0,k}(a\mid x^{(k)})}
    .
\end{equation*}
In OPE, $\pi_{\theta}$ is known, so $\bar{\pi}_{\theta,k}$ can be estimated by regressing the pseudo-outcome $\pi_{\theta}(a\mid X)$ on $X^{(k)}$.

\paragraph{Weight clipping (optional).}
For numerical stability, we may clip lag weights at a threshold $d$;
see Appendix~\ref{app:subsec:weight-clipping} for the definition and the induced bias term.

\paragraph{Lag-specific DOLCE estimator.}
Let $\hat{q}_k$ be an estimate of $q_k$, and let $\hat{w}_k$ estimate $w_k$.
We define, for each lag $k$,
\begin{align}
    \nonumber
    \hat{V}_k(\pi_{\theta})
    &:=
    \frac{1}{n}\sum_{i=1}^n
    \Bigl[
        \hat{w}_k(X_i^{(k)}, A_i)\{R_i-\hat{q}_k(X_i,X_i^{(k)},A_i)\}
        \\
    &\qquad
    \label{eq:dolce-ope-k}
        +\sum_{a\in\mathcal{A}}\pi_{\theta}(a\mid X_i)\hat{q}_k(X_i,X_i^{(k)},a)
    \Bigr]
    .
\end{align}
The final multi-lag DOLCE estimate is the softmin-aggregated version
\begin{equation}
\label{eq:dolce-ope-softmin}
    \hat{V}_{\mathrm{DOLCE}}(\pi_{\theta})
    :=
    \sum_{k=1}^K \alpha_k \hat{V}_k(\pi_\theta)
    ,
\end{equation}
where $\alpha_k$ are defined in \eqref{eq:softmin-weights} (computed with sample splitting / cross-fitting;
see Section~\ref{subsec:estimating-procedure}).

\paragraph{Bias cancellation and inference.}
We collect oracle finite-sample identities for lag-$k$ DOLCE in Appendix~\ref{app:subsec:oracle-dolce-identities}.
In particular, an exact bias identity for the oracle-weight version of \eqref{eq:dolce-ope-k} is given in Proposition~\ref{prop:bias-dolce-ope} (proof in Appendix~\ref{app:proof-bias-dolce-ope}), and a finite-sample variance decomposition is given in Proposition~\ref{prop:var-dolce-ope-oracle} (proof in Appendix~\ref{app:proof-var-dolce-ope-oracle}).
Under Assumptions~\ref{ass:current-action-sufficiency}--\ref{ass:ri}, the oracle DOLCE estimator is unbiased (Theorem~\ref{thm:unbiased-dolce-ope-oracle}).
With estimated nuisances and cross-fitting, $\hat{V}_k(\pi_\theta)$ is consistent and asymptotically normal under standard regularity and nuisance-rate conditions;
see Appendix~\ref{app:subsec:inference}, in particular Theorem~\ref{thm:an-dolce-ope}, for the influence-function representation and asymptotic variance.

Under Assumption~\ref{ass:current-action-sufficiency}, $\pi_0(a\mid X)$ is the conditional action probability given $X$, and the lag marginal $\bar{\pi}_{0,k}(a\mid X^{(k)})$ equals $\mathbb{E}_P[\pi_0(a\mid X)\mid X^{(k)}]$.
This enables exact cancellation under Assumption~\ref{ass:ri}.

\begin{theorem}[Unbiasedness under lag overlap and residual invariance]
\label{thm:unbiased-dolce-ope-oracle}
    Fix $k\in[K]$.
    Suppose Assumptions~\ref{ass:ci1:consistency}--\ref{ass:ci2:unconfoundedness}, \ref{ass:current-action-sufficiency}, and \ref{ass:common-lag-support} hold.
    If $\tilde{q}_k$ satisfies residual invariance (Assumption~\ref{ass:ri}), then the lag-$k$ oracle DOLCE estimator is unbiased:
    \begin{equation*}
        \mathbb{E}_P[\hat{V}_k(\pi_{\theta})]=V(\pi_{\theta})
        ,
    \end{equation*}
    where $\hat{V}_k$ uses $\hat{w}_k \equiv w_k$ and $\hat{q}_k \equiv \tilde{q}_k$.
\end{theorem}
See Appendix~\ref{app:proof-unbiased-dolce-ope-oracle} for the proof of Theorem~\ref{thm:unbiased-dolce-ope-oracle}.

\paragraph{Inference (estimated nuisances, cross-fitting).}
With estimated nuisances and cross-fitting, $\hat{V}_k(\pi_\theta)$ admits an influence-function expansion and is asymptotically normal under standard regularity and nuisance-rate conditions;
see Appendix~\ref{app:subsec:inference}.

\subsection{DOLCE for OPL}
\label{subsec:dolce-opl}
We present the gradient analogue of DOLCE.
For gradient-based OPL, we consider a differentiable policy class $\{\pi_{\theta}\}$ and aim to estimate $g(\theta) := \nabla_{\theta}V(\pi_{\theta})$.

\paragraph{Lag-marginal score.}
For each lag $k$, define the lag-marginal score $\bar{s}_{\theta,k}(a\mid x^{(k)}) := \nabla_{\theta}\log\bar\pi_{\theta,k}(a\mid x^{(k)})$, where $\bar{\pi}_{\theta,k}(a\mid x^{(k)}) = \mathbb{E}_{P}[\pi_{\theta}(a\mid X)\mid X^{(k)}=x^{(k)}]$.
We estimate $\bar{s}_{\theta,k}$ by regressing $\pi_{\theta}(a\mid X)s_{\theta}(a\mid X)$ and $\pi_{\theta}(a\mid X)$ on $X^{(k)}$ and taking their ratio;
see Appendix~\ref{app:subsec:dolce-opl-details} for conditions and details.

\paragraph{Lag-$k$ DOLCE gradient estimator.}
Let $\hat{q}_k$ and $\hat{w}_{\theta, k}$ be cross-fitted.
Define
\begin{align}
\nonumber
    \hat{g}_k(\theta)
    &:=
    \frac{1}{n}\sum_{i=1}^n\Bigl[
        \hat{w}_{\theta,k}(X_i^{(k)},A_i)\{R_i-\hat{q}_k(X_i,X_i^{(k)},A_i)\}
        \hat{\bar{s}}_{\theta,k}(A_i\mid X_i^{(k)})
        \\
\label{eq:dolce-pg-k}
    &\qquad
        +\sum_{a\in\mathcal{A}}\pi_{\theta}(a\mid X_i)\hat{q}_k(X_i,X_i^{(k)},a)s_{\theta}(a\mid X_i)
    \Bigr]
    .
\end{align}
The multi-lag version is obtained by softmin aggregation as in \eqref{eq:dolce-ope-softmin}.

\begin{remark}[Theoretical consequences mirror OPE]
    Under Assumptions~\ref{ass:common-lag-support} and \ref{ass:ri} (and suitable smoothness of $\pi_{\theta}$ so that $\bar{s}_{\theta,k}$ is well-defined;
    see, e.g., Assumption~\ref{ass:diff-under-ce}), together with additional nuisance-rate conditions for $\widehat{\bar{s}}_{\theta,k}$, the estimator~\eqref{eq:dolce-pg-k} is
    (i) unbiased in the oracle setting,
    (ii) consistent, and
    (iii) asymptotically normal for fixed $\theta$.
    Proofs follow from standard policy-gradient arguments and are deferred to Appendix~\ref{app:sec:proofs}.
\end{remark}

\subsection{Estimating Procedure: Moment-Targeted Residual Invariance}
\label{subsec:estimating-procedure}

\citet{saito2023off,saito2025potec} propose a two-step procedure that directly targets pairwise differences in $X$ to enforce local correctness.
This requires constructing pairwise datasets $\mathcal{D}_{\mathrm{pair}}$, which can be restrictive when exact (or near) matches in lag contexts are scarce.
We instead propose an estimator that directly targets residual invariance via moment conditions using only $\mathcal{D}_{\mathrm{lag}}$.

\paragraph{Residual invariance as an orthogonality condition.}
Fix a lag $k$ and define the $\sigma$-field $\mathcal{G}_k := \sigma(X^{(k)}, A)$.
Let $L_2(P)$ denote square-integrable functions of $(X,X^{(k)},A)$ under $P$.

\begin{proposition}[Residual invariance $\Leftrightarrow$ moment orthogonality]
\label{prop:ri-moment-equiv}
    Fix lag $k$ and let $\Delta\in L_2(P)$.
    Then $\Delta$ is $\mathcal{G}_k$-measurable (i.e., $\Delta\in L_2(\mathcal{G}_k)$) if and only if
    \begin{equation*}
        \mathbb{E}_P[\Delta f]
        =
        0
        ,
        \quad
        \forall f\in L_{2,0}(\mathcal{G}_k)
        :=
        \{f\in L_2(P): \mathbb{E}_P[f\mid \mathcal{G}_k] = 0\}
        .
    \end{equation*}
    Equivalently, $\|\Delta-\mathbb{E}_P[\Delta\mid \mathcal{G}_k]\|_{L_2(P)}=0$.
    In particular, when $\Delta=\Delta_k(X,X^{(k)},A;\tilde{q}_k)$ is a model residual, this condition is exactly Assumption~\ref{ass:ri}.
\end{proposition}
See Appendix~\ref{app:proof-ri-moment-equiv} for the proof of Proposition~\ref{prop:ri-moment-equiv}.

\paragraph{A conditional-variance objective.}
For a candidate reward model $q'\in\mathcal{Q}_k$ (a hypothesis class), define
\begin{equation*}
    \mathrm{ALC}_k(q')
    :=
    \mathbb{E}_P\Bigl[
        \mathbb{V}\bigl(
            q_k(X,X^{(k)},A) - q'(X,X^{(k)}, A)\mid \mathcal{G}_k
        \bigr)
    \Bigr]
    ,
\end{equation*}
where $\mathbb{V}(\cdot)$ denotes the variance.
Then $\mathrm{ALC}_k(q') = 0$ if and only if $q_k-q'\in L_2(\mathcal{G}_k)$, i.e., $q'$ satisfies residual invariance (Assumption~\ref{ass:ri}).

\paragraph{Moment-targeting via minimax / generalized method of moments (GMM).}
By Proposition~\ref{prop:ri-moment-equiv}, $\mathrm{ALC}_k$ can be characterized as a supremum over centered test functions.
Let $\mathcal{F}_k$ be a critic class and define the centered critic
\begin{equation*}
    \tilde{f}(X,X^{(k)},A)
    \!:=\!
    f(X,X^{(k)},A) - \mathbb{E}_P[f(X,X^{(k)},A)\mid \mathcal{G}_k]
    .
\end{equation*}
In practice, the conditional expectation is replaced by a regression estimate and cross-fitted.

Define the population objective
\begin{align}
    \label{eq:mtri-pop}
    &\mathcal{L}_k(q')
    :=
    \mathbb{E}_P[(R-q'(X,X^{(k)},A))^2]\\
    \nonumber
    &+
    \lambda
    \sup_{\substack{\tilde{f}\in\tilde{\mathcal{F}}_k,\\ \|\tilde{f}\|_{L_2(P)}\le 1}}
    \Bigl(
        \mathbb{E}_P[(R-q'(X,X^{(k)},A))\tilde{f}(X,X^{(k)},A)]
    \Bigr)^2
    ,
\end{align}
where $\tilde{\mathcal{F}}_k := \{\tilde{f}: \tilde{f} = f-\mathbb{E}_P[f\mid \mathcal{G}_k],\; f\in \mathcal{F}_k\}$ and $\lambda > 0$ is a tuning parameter.

\begin{definition}[MTRI (Moment-Targeted Residual Invariance) estimator]
\label{def:mtri}
    Fix lag $k$.
    Given $\mathcal{D}_{\mathrm{lag}}$, the MTRI estimator $\hat{q}_k$ is defined as an approximate minimizer of the empirical, cross-fitted analogue of \eqref{eq:mtri-pop}:
    \begin{equation*}
        \hat{q}_k\in\arg\min_{q'\in\mathcal{Q}_k}\widehat{\mathcal{L}}_k(q')
        ,
    \end{equation*}
    where $\widehat{\mathcal{L}}_k$ replaces $\mathbb{E}_P$ by out-of-fold empirical averages and replaces the centering $\mathbb{E}_P[\cdot\mid \mathcal{G}_k]$ by an out-of-fold regression.
\end{definition}

The minimax moment term in \eqref{eq:mtri-pop} targets residual invariance by lower bounding $\mathrm{ALC}_k$; see Appendix~\ref{app:subsec:mtri-alc}.
We use $K_{\mathrm{cf}}$-fold cross-fitting for all nuisances and construct Wald-type confidence intervals (CIs) from cross-fitted influence functions;
see Appendix~\ref{app:subsec:cross-fitting-ci}.

\subsection{Whole Algorithm}
\label{subsec:whole-algorithm}
Based on the discussion above, the full DOLCE procedure can be constructed using Algorithm~\ref{alg:dolce}, named \emph{DOLCE} with Multi-Lag Softmin and MTRI.

\begin{algorithm}[tb]
  \caption{DOLCE with Multi-Lag Softmin and MTRI}
  \label{alg:dolce}
  \begin{algorithmic}
    \REQUIRE Logged data $\mathcal{D}_{\mathrm{lag}} = \{(\boldsymbol{X}_{0i},X_i,A_i,R_i)\}_{i=1}^n$;
    lag set $\mathcal{L}=\{\ell_1,\ldots, \ell_K\}$;
    policy $\pi_\theta$ (fixed for OPE; parameterized for OPL);
    cross-fitting folds $K_{\mathrm{cf}}$;
    MTRI critic class $\mathcal{F}_k$ and model class $\mathcal{Q}_k$;
    penalty $\lambda$;
    softmin temperature $\tau$;
    weight clipping threshold $d$.
    
    \ENSURE OPE estimate $\hat{V}_{\mathrm{DOLCE}}(\pi_{\theta})$ and/or gradient estimate $\widehat{g}_{\mathrm{DOLCE}}(\theta)$.
    
    \STATE Split $[n]$ into $K_{\mathrm{cf}}$ folds $\{l_1,\ldots, l_{K_{\mathrm{cf}}}\}$.
    \FOR{$j=1,\ldots,K_{\mathrm{cf}}$}
        \STATE Let training indices $l_{-j}:=[n]\setminus l_j$.
        \FOR{$k=1,\ldots,K$}
            \STATE (Lag propensity) Fit $\widehat{\bar{\pi}}_{0,k}^{(-j)}(a\mid x^{(k)})\approx P(A=a\mid X^{(k)}=x^{(k)})$ on $l_{-j}$.
            \STATE (Lag target marginal)
            \IF{OPE}
                \STATE Fit $\widehat{\bar{\pi}}_{\theta,k}^{(-j)}(a\mid x^{(k)}) \approx \mathbb{E}[\pi_{\theta}(a\mid X)\mid X^{(k)} = x^{(k)}]$ on $l_{-j}$.
            \ELSIF{OPL}
                \STATE Fit $\widehat{\bar{\pi}}_{\theta,k}^{(-j)}$ and $\widehat{m}_{\theta,k}^{(-j)}(a\mid x^{(k)}) \approx \mathbb{E}[\pi_{\theta}(a\mid X)s_{\theta}(a\mid X)\mid X^{(k)}]$ on $l_{-j}$.
            \ENDIF
        \ENDFOR
    \ENDFOR
  \end{algorithmic}
\end{algorithm}
\section{Experiments}
\label{sec:experiments}

\subsection{Synthetic Data Experiments}
\label{subsec:syn-experiments}

We begin with controlled synthetic contextual bandits where
(i) the true policy value is known up to Monte Carlo error and
(ii) we can continuously increase current-context overlap violations while preserving overlap at the lag level.
This setting isolates the failure mode of standard OPE/OPL methods under support violations (Section~\ref{subsec:support-violation}) and directly tests the bias-cancellation mechanism of DOLCE (Section~\ref{sec:proposed-method}).

\subsubsection{Data Generating Process}
\label{subsubsec:syn-dgp}
We generate i.i.d.\ tuples $(X^{(1)}, X, A, R)$ with one lag ($K=1$).
The lag context is sampled as $X^{(1)}\in \mathbb{R}^d\sim N(0,I_d)$ and the current context is sampled as $X\sim N(X^{(1)}, 3^2 I_d)$.
To create overlap violations only at the current context while maintaining lag overlap, we overwrite the first coordinate so that $X_1\sim N(0,3^2)$ independently of $X^{(1)}$.

Rewards follow a mixture of current and lag effects:
\begin{equation*}
    q(X,X^{(1)},a)
    =
    \lambda g(X,a)
    +
    (1-\lambda)h(X^{(1)},a)
    +
    \eta u(X,X^{(1)},a)
    ,
\end{equation*}
where $g$ and $h$ are piecewise-constant functions of thresholded features and $u$ is an optional interaction term used to control violations of residual invariance ($\eta = 0$ yields an additive structure).
Observed rewards are $R = q(X,X^{(1)},A) + \varepsilon$ with $\varepsilon\sim N(0,1)$.

The logging policy depends only on the current context (current-action sufficiency): $\pi_0(a\mid X)\propto\exp\{\beta g(X,a)\}$.
To impose current-context support violations, for the top $r$ fraction of samples in $X_1$ (where $r$ is the support violation ratio), we deterministically set $A=0$ and enforce $\pi_0(0\mid X) = 1$ and $\pi_0(a\mid X)=0$ for all $a\ne 0$.
Because $X_1$ is independent of $X^{(1)}$, this construction induces severe violations of $\pi_0(\cdot\mid X)$ while keeping $\bar{\pi}_{0,1}(\cdot\mid X^{(1)})$ supported for all actions as long as $r<1$.

Unless otherwise stated, we use the default settings:
number of logged samples $n=1,000$, number of actions $|\mathcal{A}| = 5$, feature dimension $d=10$, support violation ratio $r=50\%$, mixture parameter $\lambda = 0.5$, interaction strength $\eta = 0$, and logging-temperature $\beta = 0.3$.
We report additional implementation details and the exact functional forms of $g$, $h$, $u$ in Appendix~\ref{app:subsubsec:syn-dgp}.

\subsubsection{Evaluation Metrics}
\label{subsubsec:syn-metrics}

\paragraph{OPE.}
For each method, we evaluate the mean-squared error (MSE) of the estimated policy value:
\begin{equation*}
    \mathrm{MSE}
    :=
    \mathbb{E}\Bigl[
        \bigl(
            \widehat{V} - V(\pi_\theta)
        \bigr)^2
    \Bigr]
    ,
\end{equation*}
approximated by Monte Carlo averaging over repeated datasets.
We also evaluate coverage of nominal $95\%$ confidence intervals as a function of the support violation ratio $r$.
For each run, we form a Wald-type interval $\widehat{V}\pm z_{0.975}\widehat{\mathrm{SE}}$, $\widehat{\mathrm{SE}} = \sqrt{{n^{-2}}\sum_{i=1}^n \hat{\phi}_i^2}$, where $\hat{\phi}_i$ is the cross-fitted influence-function estimate (equivalently, the per-sample contribution centered by $\widehat{V}$).
Coverage is the fraction of runs in which the interval contains the true value $V(\pi_\theta)$.
We use $100$ Monte Carlo replications per configuration.
Sensitivity results for bias/variance/MSE/coverage when varying $\lambda$, $|\mathcal{A}|$, $n$, and $\eta$ are reported in Appendix~\ref{app:subsubsec:syn-ope-sensitivity}.

\paragraph{OPL.}
For OPL we evaluate learned policies on a large, independently generated test set where the true reward function $q$ is known.
Let $\widehat{\pi}$ be the learned policy and let $\pi_0$ be the logging policy.
Define the oracle best-in-class value (on the test distribution) as $V^{\star} := \mathbb{E}[\max_{a\in\mathcal{A}}q(X,X^{(1)},a)]$.
We report:
\begin{equation*}
    \begin{split}
        \text{Normalized Improvement (NI)}
        &:=
        \frac{V(\widehat{\pi}) - V(\pi_0)}{V^\star - V(\pi_0)}
        ,
        \\
        \text{One-step Improvement (OSI)}
        &:=
        V(\pi_{\theta_0+\alpha \hat{g}})-V(\pi_{\theta_0})
        ,
        \\
        \text{Regret}
        &:=
        V^\star - V(\widehat{\pi})
        ,
    \end{split}
\end{equation*}
where $\theta_0$ is a common policy-network initialization, $\hat{g}$ is the estimated policy gradient (IPS/DR/DOLCE), and $\alpha$ is the step size used by the optimizer.
NI measures value gain relative to the logging policy after normalizing by the achievable gap, OSI isolates the quality of the learned gradient direction, and regret quantifies distance to the oracle greedy policy.
We use $50$ Monte Carlo replications per configuration.
Additional sensitivity results (NI/OSI/regret) for varying $\lambda$, $|\mathcal{A}|$, $n$ and $\eta$ are provided in Appendix~\ref{app:subsubsec:syn-opl-sensitivity}.

\subsubsection{Results}
\label{subsubsec:syn-results}

\paragraph{OPE under support violation.}
Figure~\ref{fig:syn-ope} (a) shows that as the support violation ratio $r$ increases, the MSE of DM, IPS, and DR rises sharply with deteriorating overlap, while DOLCE stays near zero over most of the range, indicating robust lag-overlap bias cancellation;
only at very large $r$ does DOLCE exhibit a slight MSE increase, consistent with mild variance inflation under more extreme reweighting (Figure~\ref{fig:syn-ope} (c)).
DM and DR remain competitive under moderate violations, whereas IPS is consistently poor, aligning with Lemma~\ref{lem:bias-support}:
IPS incurs unavailable bias proportional to unsupported target mass, while DM/DR can remain accurate when the reward model is (approximately) correct on relevant regions, including parts of the unsupported set.
Figure~\ref{fig:syn-ope} (b) shows that DOLCE is essentially unbiased across $r$, whereas IPS bias grows with $r$ and DM/DR become sensitive to reward-model misspecification on unsupported actions as overlap worsens.
Figure~\ref{fig:syn-ope} (c) shows that DOLCE’s variance increases only in the high-violation regime, consistent with Proposition~\ref{prop:var-dolce-ope-oracle} and explaining the small MSE increase in (a).
Figure~\ref{fig:syn-ope} (d) shows that $95\%$ CI coverage for IPS/DM/DR becomes unstable and drops as violations intensify due to non-negligible bias unaccounted for by standard variance estimates, while DOLCE maintains near-nominal coverage via influence-function-based inference enabled by lag-based unbiasedness.

\begin{figure*}[tb]
\vskip 0.2in
\begin{center}
\centerline{\includegraphics[width=\columnwidth]{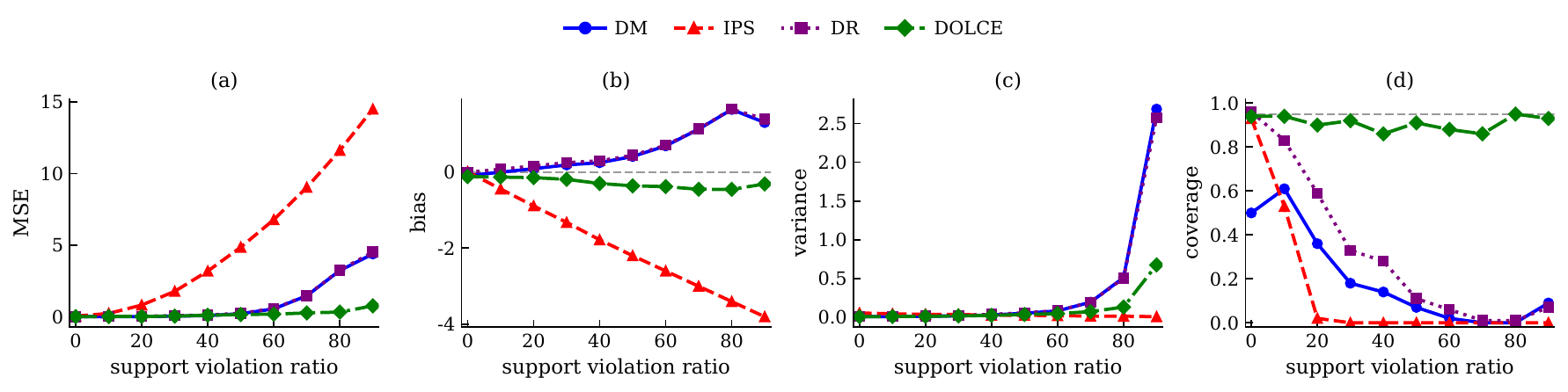}}
\caption{
Off-policy evaluation under support violation.
(a) Mean squared error (MSE) of the estimated policy value $\hat{V}(\pi)$ as the support violation ratio $r$ increases.
(b) Bias of $\hat{V}(\pi)$ versus $r$, where the horizontal gray dashed line corresponds to zero bias.
(c) Variance of $\hat{V}(\pi)$ versus $r$.
(d) Empirical coverage of nominal $95\%$ confidence intervals for $V(\pi)$ as a function of $r$, where the horizontal gray dashed line denotes the nominal $95\%$ coverage level.
}
\label{fig:syn-ope}
\end{center}
\vskip -0.2in
\end{figure*}

\paragraph{OPL under support violation.}
Figure~\ref{fig:syn-opl} summarizes OPL performance as $r$ increases.
In Figure~\ref{fig:syn-opl} (a), DOLCE yields larger normalized improvements in the high-violation regime, indicating that the learned policy tends to improve more reliably over the logging policy when current-context overlap is poor.
Figure~\ref{fig:syn-opl} (b) further shows that DOLCE attains higher one-step improvements, suggesting that its lag-marginal gradient correction better aligns the update direction with the true value gradient under support violations.
Finally, Figure~\ref{fig:syn-opl} (c) reports regret:
while IPS/DR-based learners exhibit increasing regret as $r$ grows, DOLCE remains comparatively stable, indicating reduced degradation of optimization quality when overlap collapses.
At the same time, DOLCE does not uniformly dominate all baselines across all regimes (e.g., when $r$ is small, standard methods can be competitive);
rather, the main takeaway is that DOLCE more consistently steers OPL in a favorable direction under severe support violations.

\begin{figure*}[tb]
\vskip 0.2in
\begin{center}
\centerline{\includegraphics[width=\columnwidth]{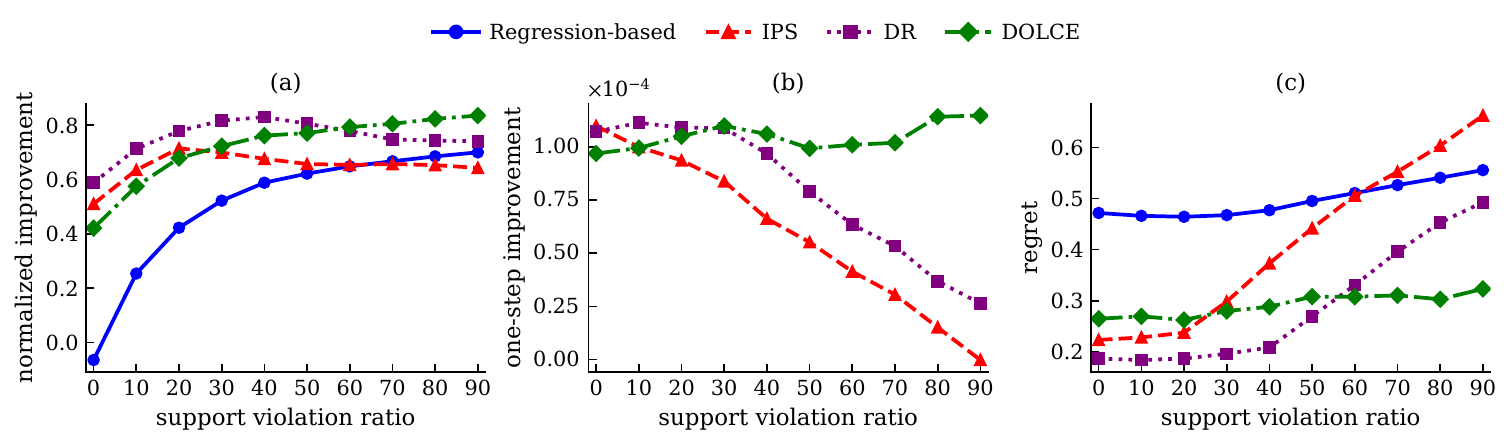}}
\caption{
Off-policy learning under support violation.
(a) Normalized improvement over the logging policy as a function of the support violation ratio $r$.
(b) One-step improvement versus $r$, reflecting the alignment between the policy update direction and the true value gradient.
(c) Regret as $r$ increases.
}
\label{fig:syn-opl}
\end{center}
\vskip -0.2in
\end{figure*}

\subsection{Real-World Experiments}
\label{subsec:real-experiments}

\subsubsection{Dataset}
\label{subsubsec:dataset}
We conduct real-world experiments on the Northwestern ICU dataset (NWICU) \citep{nwicu010}, a deidentified electronic health record (EHR) database that contains ICU stay information together with time-stamped bedside charting and medication administration records.
NWICU provides repeated measurements over time within an ICU stay, naturally matching the lagged contextual bandit formulation in Section~\ref{subsec:lagged-formulation}.
We restrict to ICU stays with valid admission/discharge timestamps and sufficient vital sign measurements;
full cohort construction is described in Appendix~\ref{app:subsec:realworld-nwicu}.

\subsubsection{Contextual Bandit Construction}
\label{subsubsec:contextual-bandit-construction}
We discretize each ICU stay into hourly decision points ($\Delta t=1$h).
At each decision time $t$, the current context $X_t$ summarizes routinely charted vitals (HR, SpO$_2$, SBP, DBP) and MAP computed as $\mathrm{MAP}=(\mathrm{SBP}+2\mathrm{DBP})/3$.
We construct $K=3$ lag contexts at $\{1,2,4\}$ hours and drop decision points without all lags; details (item IDs, range filters, last-observation-carried-forward within a 2-hour tolerance, and window alignment) are given in Appendix~\ref{app:subsec:realworld-nwicu}.

The action $A_t\in\{0,1\}$ indicates whether any vasopressor is administered within $[t,t+1\mathrm{h})$, identified from eMAR medication names.
The reward is a short-horizon hemodynamic proxy $R_t=\mathbf{1}\{\mathrm{MAP}_{t+1\mathrm{h}}\ge 65\}$.
To better match the i.i.d.\ bandit setup, we subsample one decision point per ICU stay before splitting train/evaluation sets, so the effective sample size equals the number of unique ICU stays (Appendix~\ref{app:subsec:realworld-nwicu}).

\subsubsection{Target Policies and Overlap Stress}
\label{subsubsec:realworld-target-policy}
Because the counterfactual value is unobservable in NWICU, we evaluate clinically motivated stochastic target policies that deviate from clinician behavior.
Motivated by common hemodynamic targets around MAP $\approx 65$\,mmHg and vasopressor use in hypotension \citep{green20132012,white2025mean}, we use a heuristic rule-based stochastic policy: $\pi_{\theta}(A_t=1\mid X_t) = \sigma\{(65-\mathrm{MAP}_t)/5 + (\mathrm{HR}_t-110)/30\}$, where $\sigma(\cdot)$ is the logistic sigmoid and probabilities are clipped away from $0$ and $1$ for numerical stability.
The HR term and scaling constants are chosen heuristically to produce smooth probability changes around clinically relevant ranges.

\subsubsection{Evaluation Protocol and Baselines}
\label{subsubsec:realworld-evaluation}
For OPE, we report $\widehat{V}(\pi_\theta)$ with Wald-type confidence intervals using cross-fitted nuisance estimates.
We compare DM, IPS, DR, and DOLCE, and report ESS for importance-weighted estimators.
For OPL, we learn policies by optimizing each objective on the training split, yielding $\hat{\pi}_{\mathrm{DM}},\hat{\pi}_{\mathrm{IPS}},\hat{\pi}_{\mathrm{DR}},\hat{\pi}_{\mathrm{DOLCE}}$, and evaluate them on the held-out split via OPE (Appendix~\ref{app:subsec:realworld-nwicu}).

\subsubsection{Results}
\label{subsubsec:realworld-results}

Table~\ref{tab:nwicu-ope} reports OPE results for the vasopressor target policy.
All estimators yield values around $0.95$, indicating that the one-hour MAP threshold is frequently achieved in this cohort and horizon (a ceiling-effect regime for further improvements).
IPS/DR exhibit larger uncertainty than DM, consistent with variance inflation from reweighting under limited overlap; this is reflected by reduced ESS.

Table~\ref{tab:nwicu-opl} summarizes OPL outcomes by evaluating learned policies with the DOLCE OPE estimator.
Across learned policies, estimated values are close and confidence intervals overlap, suggesting limited detectable differences under the current reward/horizon definition.
Notably, IPS-based evaluation can be unstable (e.g., producing out-of-range estimates for binary rewards in some settings; see Appendix~\ref{app:subsec:realworld-nwicu:opl-full}), reinforcing the importance of overlap-aware estimators and diagnostics when comparing learned policies offline.

\begin{table}[tb]
\centering
\caption{
OPE results on NWICU for the vasopressor target policy.
ESS denotes the effective sample size computed from realized importance weights for IPS/DR (and reported as $n$ for non-importance-weighted estimators in this table).
}
\label{tab:nwicu-ope}
\begin{tabular}{lrrr}
\toprule
Estimator & $\widehat{V}$ & 95\%CI & ESS \\
\midrule
DM    & $0.953$ & $[0.951,0.955]$ & $16043.0$ \\
IPS   & $0.954$ & $[0.949,0.959]$ & $13675.9$ \\
DR    & $0.953$ & $[0.948,0.958]$ & $13675.9$ \\
DOLCE & $0.953$ & $[0.950,0.957]$ & $16043.0$ \\
\bottomrule
\end{tabular}
\end{table}

\begin{table}[t]
\centering
\caption{
OPL results on NWICU: value of learned policies evaluated by DOLCE on the held-out split.
}
\label{tab:nwicu-opl}
\begin{tabular}{lrr}
\toprule
Learned policy & $\widehat{V}$ & 95\%CI \\
\midrule
$\hat{\pi}_{\mathrm{DM}}$    & $0.944$ & $[0.906,0.981]$ \\
$\hat{\pi}_{\mathrm{IPS}}$   & $0.946$ & $[0.926,0.967]$ \\
$\hat{\pi}_{\mathrm{DR}}$    & $0.953$ & $[0.945,0.962]$ \\
$\hat{\pi}_{\mathrm{DOLCE}}$ & $0.949$ & $[0.937,0.962]$ \\
\bottomrule
\end{tabular}
\end{table}
\section{Conclusion}
\label{sec:conclusion}

We proposed \textbf{DOLCE}, a lag-aware framework for OPE and OPL under current-context support violations.
DOLCE replaces current-context importance weights with lag-marginal ratios and achieves exact bias cancellation under lag overlap and a residual-invariance condition, while enabling asymptotically valid inference via cross-fitting.
We further introduced a practical estimation strategy (MTRI) that targets residual invariance through moment conditions and a softmin aggregation mechanism to implement a one-good-lag principle in multi-lag settings.
In synthetic experiments, DOLCE substantially improves OPE robustness and maintains near-nominal confidence interval coverage as overlap deteriorates, and it produces more reliable optimization behavior in OPL under severe support violations.
\section*{Code availability}
The Python implementation of proposed method and simulation experiments in this study is available at \href{https://github.com/shutech2001/DOLCE}{\texttt{https://github.com/shutech2001/DOLCE}}.
\section*{Acknowledgements}
\label{sec:acknowledgements}
We would like to thank Masanori Nojima for their valuable comments on this article.
Shu Tamano was supported by JSPS KAKENHI Grant Numbers 25K24203.

\bibliographystyle{apalike}
\bibliography{bibliography}

@article{balke1997bounds,
    author    = {Balke, A. and Pearl, J.},
    journal   = {Journal of the American Statistical Association},
    number    = {439},
    pages     = {1171--1176},
    publisher = {Taylor \& Francis},
    title     = {Bounds on treatment effects from studies with imperfect compliance},
    volume    = {92},
    year      = {1997}
}

@article{bottou2013counterfactual,
    author    = {Bottou, L{\'e}on and Peters, J. and Qui{\~n}onero-Candela, J. and Charles, D.~X. and Chickering, D.~M. and Portugaly, E. and Ray, D. and Simard, P. and Snelson, E.},
    journal   = {Journal of Machine Learning Research},
    number    = {101},
    pages     = {3207--3260},
    title     = {Counterfactual reasoning and learning systems: The example of computational advertising},
    volume    = {14},
    year      = {2013}
}

@article{chernozhukov2018double,
    author    = {Chernozhukov, V. and Chetverikov, D. and Demirer, M. and Duflo, E. and Hansen, C. and Newey, W. and Robins, J.},
    journal   = {The Econometrics Journal},
    number    = {1},
    pages     = {C1--C68},
    title     = {Double/debiased machine learning for treatment and structural parameters},
    volume    = {21},
    year      = {2018}
}

@article{crump2009dealing,
    author    = {Crump, R.~K. and Hotz, V.~J. and Imbens, G.~W. and Mitnik, O.~A.},
    journal   = {Biometrika},
    number    = {1},
    pages     = {187--199},
    publisher = {Oxford University Press},
    title     = {Dealing with limited overlap in estimation of average treatment effects},
    volume    = {96},
    year      = {2009}
}

@article{d2021overlap,
    author    = {D’Amour, A. and Ding, P. and Feller, A. and Lei, L. and Sekhon, J.},
    journal   = {Journal of Econometrics},
    number    = {2},
    pages     = {644--654},
    publisher = {Elsevier},
    title     = {Overlap in observational studies with high-dimensional covariates},
    volume    = {221},
    year      = {2021}
}

@inproceedings{dudik2011doubly,
    author    = {Dud{\'i}k, M. and Langford, J. and Li, L.},
    booktitle = {Proceedings of the 28th International Conference on Machine Learning},
    pages     = {1097--1104},
    title     = {Doubly robust policy evaluation and learning},
    year      = {2011}
}

@inproceedings{farajtabar2018more,
    author    = {Farajtabar, M. and Chow, Y. and Ghavamzadeh, M.},
    booktitle = {Proceedings of the 35th International Conference on Machine Learning},
    pages     = {1447--1456},
    publisher = {PMLR},
    title     = {More robust doubly robust off-policy evaluation},
    year      = {2018}
}

@article{felicioni2022off,
    author    = {Felicioni, N.~{\`o} and Ferrari Dacrema, M. and Restelli, M. and Cremonesi, P.},
    journal   = {Advances in Neural Information Processing Systems},
    pages     = {30250--30264},
    title     = {Off-policy evaluation with deficient support using side information},
    volume    = {35},
    year      = {2022}
}

@inproceedings{fujimoto2019off,
    author    = {Fujimoto, S. and Meger, D. and Precup, D.},
    booktitle = {Proceedings of the 36th International Conference on Machine Learning},
    pages     = {2052--2062},
    publisher = {PMLR},
    title     = {Off-policy deep reinforcement learning without exploration},
    year      = {2019}
}

@inproceedings{gui2022causal,
    author    = {Gui, L. and Veitch, V.},
    booktitle = {Proceedings of the 11th International Conference on Learning Representations},
    title     = {Causal estimation for text data with (apparent) overlap violations},
    year      = {2023}
}

@article{hill2013assessing,
    author    = {Hill, J. and Su, Y.-S.},
    journal   = {Annals of Applied Statistics},
    number    = {3},
    pages     = {1386--1420},
    title     = {Assessing lack of common support in causal inference using {Bayesian} nonparametrics: Implications for evaluating the effect of breastfeeding on children's cognitive outcomes},
    volume    = {7},
    year      = {2013}
}

@inproceedings{jiang2016doubly,
    author    = {Jiang, N. and Li, L.},
    booktitle = {Proceedings of the 33rd International Conference on Machine Learning},
    pages     = {652--661},
    publisher = {PMLR},
    title     = {Doubly robust off-policy value evaluation for reinforcement learning},
    year      = {2016}
}

@article{jin2025policy,
    author    = {Jin, Y. and Ren, Z. and Yang, Z. and Wang, Z.},
    journal   = {Annals of Statistics},
    number    = {4},
    pages     = {1483--1512},
    publisher = {Institute of Mathematical Statistics},
    title     = {Policy learning ``without'' overlap: Pessimism and generalized empirical bernstein’s inequality},
    volume    = {53},
    year      = {2025}
}

@inproceedings{joachims2018deep,
    author    = {Joachims, T. and Swaminathan, A. and De Rijke, M.},
    booktitle = {Proceedings of the 6th International Conference on Learning Representations},
    title     = {Deep learning with logged bandit feedback},
    year      = {2018}
}

@article{kallus2020double,
    author    = {Kallus, N. and Uehara, M.},
    journal   = {Journal of Machine Learning Research},
    number    = {167},
    pages     = {1--63},
    title     = {Double reinforcement learning for efficient off-policy evaluation in {Markov} decision processes},
    volume    = {21},
    year      = {2020}
}

@inproceedings{kallus2021optimal,
    author    = {Kallus, N. and Saito, Y. and Uehara, M.},
    booktitle = {Proceedings of the 38th International Conference on Machine Learning},
    pages     = {5247--5256},
    publisher = {PMLR},
    title     = {Optimal off-policy evaluation from multiple logging policies},
    year      = {2021}
}

@article{kennedy2019nonparametric,
    author    = {Kennedy, E.~H.},
    journal   = {Journal of the American Statistical Association},
    number    = {526},
    pages     = {645--656},
    publisher = {Taylor \& Francis},
    title     = {Nonparametric causal effects based on incremental propensity score interventions},
    volume    = {114},
    year      = {2019}
}

@inproceedings{khan2024off,
    author    = {Khan, S. and Saveski, M. and Ugander, J.},
    booktitle = {Proceedings of the 41st International Conference on Machine Learning},
    pages     = {23734--23757},
    publisher = {PMLR},
    title     = {Off-policy evaluation beyond overlap: Sharp partial identification under smoothness},
    year      = {2024}
}

@inproceedings{kiyohara2022doubly,
    author    = {Kiyohara, H. and Saito, Y. and Matsuhiro, T. and Narita, Y. and Shimizu, N. and Yamamoto, Y.},
    booktitle = {Proceedings of the 15th International Conference on Web Search and Data Mining},
    pages     = {487--497},
    publisher = {ACM},
    title     = {Doubly robust off-policy evaluation for ranking policies under the cascade behavior model},
    year      = {2022}
}

@inproceedings{kiyohara2023off,
    author    = {Kiyohara, H. and Uehara, M. and Narita, Y. and Shimizu, N. and Yamamoto, Y. and Saito, Y.},
    booktitle = {Proceedings of the 29th ACM SIGKDD Conference on Knowledge Discovery and Data Mining},
    pages     = {1154--1163},
    title     = {Off-policy evaluation of ranking policies under diverse user behavior},
    year      = {2023}
}

@inproceedings{kiyohara2024towards,
    author    = {Kiyohara, H. and Kishimoto, R. and Kawakami, K. and Kobayashi, K. and Nakata, K. and Saito, Y.},
    booktitle = {Proceedings of the 12th International Conference on Learning Representations},
    title     = {Towards assessing and benchmarking risk-return tradeoff of off-policy evaluation},
    year      = {2024}
}

@article{kumar2019stabilizing,
    author    = {Kumar, A. and Fu, J. and Soh, M. and Tucker, G. and Levine, S.},
    journal   = {Advances in Neural Information Processing Systems},
    title     = {Stabilizing off-policy {Q-learning} via bootstrapping error reduction},
    volume    = {32},
    year      = {2019}
}

@article{kumar2020conservative,
    author    = {Kumar, A. and Zhou, A. and Tucker, G. and Levine, S.},
    journal   = {Advances in Neural Information Processing Systems},
    pages     = {1179--1191},
    title     = {Conservative {Q-learning} for offline reinforcement learning},
    volume    = {33},
    year      = {2020}
}

@inproceedings{li2010contextual,
    author    = {Li, L. and Chu, W. and Langford, J. and Schapire, R.~E.},
    booktitle = {Proceedings of the 19th International Conference on World Wide Web},
    pages     = {661--670},
    title     = {A contextual-bandit approach to personalized news article recommendation},
    year      = {2010}
}

@inproceedings{li2011unbiased,
    author    = {Li, L. and Chu, W. and Langford, J. and Wang, X.},
    booktitle = {Proceedings of the 4th ACM International Conference on Web Search and Data Mining},
    pages     = {297--306},
    publisher = {ACM},
    title     = {Unbiased offline evaluation of contextual-bandit-based news article recommendation algorithms},
    year      = {2011}
}

@article{liao2021off,
    author    = {Liao, P. and Klasnja, P. and Murphy, S.},
    journal   = {Journal of the American Statistical Association},
    number    = {533},
    pages     = {382--391},
    publisher = {Taylor \& Francis},
    title     = {Off-policy estimation of long-term average outcomes with applications to mobile health},
    volume    = {116},
    year      = {2021}
}

@article{liu2018breaking,
    author    = {Liu, Q. and Li, L. and Tang, Z. and Zhou, D.},
    journal   = {Advances in Neural Information Processing Systems},
    pages     = {5361--5371},
    title     = {Breaking the curse of horizon: Infinite-horizon off-policy estimation},
    volume    = {31},
    year      = {2018}
}

@inproceedings{liu2020understanding,
    author    = {Liu, Y. and Bacon, P.-L. and Brunskill, E.},
    booktitle = {Proceedings of the 37th International Conference on Machine Learning},
    pages     = {6184--6193},
    publisher = {PMLR},
    title     = {Understanding the curse of horizon in off-policy evaluation via conditional importance sampling},
    year      = {2020}
}

@article{manski1990nonparametric,
    author    = {Manski, C.~F.},
    journal   = {The American Economic Review},
    number    = {2},
    pages     = {319--323},
    publisher = {JSTOR},
    title     = {Nonparametric bounds on treatment effects},
    volume    = {80},
    year      = {1990}
}

@article{metelli2021subgaussian,
    author    = {Metelli, A.~M. and Russo, A. and Restelli, M.},
    journal   = {Advances in Neural Information Processing Systems},
    pages     = {8119--8132},
    title     = {Subgaussian and differentiable importance sampling for off-policy evaluation and learning},
    volume    = {34},
    year      = {2021}
}

@article{neyman1923sur,
    author    = {Neyman, J.},
    journal   = {Roczniki Nauk Rolniczych},
    number    = {1},
    pages     = {1--51},
    title     = {Sur les applications de la th\'{e}orie des probabilit\'{e}s aux exp\'{e}riences agricoles: Essai des principes},
    volume    = {10},
    year      = {1923}
}

@article{petersen2012diagnosing,
    author    = {Petersen, M.~L. and Porter, K.~E. and Gruber, S. and Wang, Y. and van der Laan, M.~J.},
    journal   = {Statistical Methods in Medical Research},
    number    = {1},
    pages     = {31--54},
    title     = {Diagnosing and responding to violations in the positivity assumption},
    volume    = {21},
    year      = {2012}
}

@article{qian2011performance,
    author    = {Qian, M. and Murphy, S.~A.},
    journal   = {Annals of Statistics},
    number    = {2},
    pages     = {1180--1210},
    title     = {Performance guarantees for individualized treatment rules},
    volume    = {39},
    year      = {2011}
}

@article{rubin1974estimating,
    author    = {Rubin, D.~B.},
    journal   = {Journal of Educational Psychology},
    number    = {5},
    pages     = {688--701},
    publisher = {American Psychological Association},
    title     = {Estimating causal effects of treatments in randomized and nonrandomized studies},
    volume    = {66},
    year      = {1974}
}

@inproceedings{sachdeva2020off,
    author    = {Sachdeva, N. and Su, Y. and Joachims, T.},
    booktitle = {Proceedings of the 26th ACM SIGKDD International Conference on Knowledge Discovery \& Data Mining},
    pages     = {965--975},
    title     = {Off-policy bandits with deficient support},
    year      = {2020}
}

@inproceedings{saito2021counterfactual,
    author    = {Saito, Y. and Joachims, T.},
    booktitle = {Proceedings of the 15th ACM Conference on Recommender Systems},
    pages     = {828--830},
    title     = {Counterfactual learning and evaluation for recommender systems: Foundations, implementations, and recent advances},
    year      = {2021}
}

@inproceedings{saito2022off,
    author    = {Saito, Y. and Joachims, T.},
    booktitle = {Proceedings of the 39th International Conference on Machine Learning},
    pages     = {19089--19122},
    publisher = {PMLR},
    title     = {Off-policy evaluation for large action spaces via embeddings},
    year      = {2022}
}

@inproceedings{saito2023off,
    author    = {Saito, Y. and Ren, Q. and Joachims, T.},
    booktitle = {Proceedings of the 40th International Conference on Machine Learning},
    pages     = {29734--29759},
    publisher = {PMLR},
    title     = {Off-policy evaluation for large action spaces via conjunct effect modeling},
    year      = {2023}
}

@inproceedings{saito2025potec,
    author    = {Saito, Y. and Yao, J. and Joachims, T.},
    booktitle = {Proceedings of the 13th International Conference on Learning Representations},
    title     = {{POTEC}: Off-policy learning for large action spaces via two-stage policy decomposition},
    year      = {2025}
}

@inproceedings{su2019cab,
    author    = {Su, Y. and Wang, L. and Santacatterina, M. and Joachims, T.},
    booktitle = {Proceedings of the 36th International Conference on Machine Learning},
    pages     = {6005--6014},
    publisher = {PMLR},
    title     = {{CAB}: Continuous adaptive blending for policy evaluation and learning},
    year      = {2019}
}

@inproceedings{su2020doubly,
    author    = {Su, Y. and Dimakopoulou, M. and Krishnamurthy, A. and Dud{\'i}k, M.},
    booktitle = {Proceedings of the 37th International Conference on Machine Learning},
    pages     = {9167--9176},
    publisher = {PMLR},
    title     = {Doubly robust off-policy evaluation with shrinkage},
    year      = {2020}
}

@article{swaminathan2015batch,
    author    = {Swaminathan, A. and Joachims, T.},
    journal   = {Journal of Machine Learning Research},
    number    = {52},
    pages     = {1731--1755},
    title     = {Batch learning from logged bandit feedback through counterfactual risk minimization},
    volume    = {16},
    year      = {2015}
}

@inproceedings{swaminathan2015counterfactual,
    author    = {Swaminathan, A. and Joachims, T.},
    booktitle = {Proceedings of the 32nd International Conference on Machine Learning},
    pages     = {814--823},
    publisher = {PMLR},
    title     = {Counterfactual risk minimization: Learning from logged bandit feedback},
    year      = {2015}
}

@article{swaminathan2017off,
    author    = {Swaminathan, A. and Krishnamurthy, A. and Agarwal, A. and Dudik, M. and Langford, J. and Jose, D. and Zitouni, I.},
    journal   = {Advances in Neural Information Processing Systems},
    title     = {Off-policy evaluation for slate recommendation},
    volume    = {30},
    year      = {2017}
}

@inproceedings{thomas2016data,
    author    = {Thomas, P. and Brunskill, E.},
    booktitle = {Proceedings of the 33rd International Conference on Machine Learning},
    pages     = {2139--2148},
    publisher = {PMLR},
    title     = {Data-efficient off-policy policy evaluation for reinforcement learning},
    year      = {2016}
}

@inproceedings{wang2017optimal,
    author    = {Wang, Y.-X. and Agarwal, A. and Dud{\'i}k, M.},
    booktitle = {Proceedings of the 34th International Conference on Machine Learning},
    pages     = {3589--3597},
    publisher = {PMLR},
    title     = {Optimal and adaptive off-policy evaluation in contextual bandits},
    year      = {2017}
}

@article{xie2019towards,
    author    = {Xie, T. and Ma, Y. and Wang, Y.-X.},
    journal   = {Advances in Neural Information Processing Systems},
    pages     = {9668--9678},
    title     = {Towards optimal off-policy evaluation for reinforcement learning with marginalized importance sampling},
    volume    = {32},
    year      = {2019}
}

@article{yang2018asymptotic,
    author    = {Yang, S. and Ding, P.},
    journal   = {Biometrika},
    number    = {2},
    pages     = {487--493},
    publisher = {Oxford University Press},
    title     = {Asymptotic inference of causal effects with observational studies trimmed by the estimated propensity scores},
    volume    = {105},
    year      = {2018}
}

@inproceedings{zhao2024positivity,
    author    = {Zhao, P. and Chambaz, A. and Josse, J. and Yang, S.},
    booktitle = {International Conference on Artificial Intelligence and Statistics},
    pages     = {1918--1926},
    publisher = {PMLR},
    title     = {Positivity-free policy learning with observational data},
    year      = {2024}
}

@article{zhu2021core,
    author    = {Zhu, Y. and Hubbard, R.~A. and Chubak, J. and Roy, J. and Mitra, N.},
    journal   = {Pharmacoepidemiology and Drug Safety},
    number    = {11},
    pages     = {1471--1485},
    publisher = {Wiley Online Library},
    title     = {Core concepts in pharmacoepidemiology: Violations of the positivity assumption in the causal analysis of observational data: Consequences and statistical approaches},
    volume    = {30},
    year      = {2021}
}

@article{nwicu010,
    author    = {Moukheiber, Dana and Temps, William and Molgi, Bhadrappa and Li, Yikuan and Lu, Alice and Nannapaneni, Prasanth and Chahin, Abdulrahman and Hao, Sicheng and {Torres Fabregas}, Felipe and Celi, Leo Anthony and Wong, Adrian and Lloyd, Maxwell and {Borrat Frigola}, Xavier and Lee, Hyung-Chul and Schneider, Daniel and Pollard, Tom and Luo, Yuan and Kho, Abel and Mark, Roger},
    title     = {{Northwestern ICU (NWICU) database}},
    journal   = {{PhysioNet}},
    year      = {2024},
    month     = nov,
    note      = {Version 0.1.0},
    doi       = {10.13026/s84w-1829},
    url       = {https://doi.org/10.13026/s84w-1829}
}

@article{white2025mean,
    title     = {Mean arterial pressure in critically ill adults receiving vasopressors: A multicentre, observational study},
    author    = {White, Kyle C and Quick, Lachlan and Durkin, Zachary and McCullough, James and Laupland, Kevin B and Blank, Sebastiaan and Attokaran, Antony G and Kumar, Aashish and Shekar, Kiran and Garrett, Peter and others},
    journal   = {Critical Care and Resuscitation},
    volume    = {27},
    number    = {1},
    pages     = {100103},
    year      = {2025},
    publisher = {Elsevier}
}

@article{green20132012,
    title     = {The 2012 Surviving Sepsis Campaign: Management of Severe Sepsis and Septic Shock—an update on the guidelines for initial therapy},
    author    = {Green, Jeffrey P and Adams, Jason and Panacek, Edward A and Albertson, Timothy A},
    journal   = {Current Emergency and Hospital Medicine Reports},
    volume    = {1},
    number    = {3},
    pages     = {154--171},
    year      = {2013},
    publisher = {Springer}
}

\newpage
\appendix
\section{Related Work}
\label{app:sec:related-work}

\subsection{Off-Policy Evaluation and Learning in Contextual Bandits and Reinforcement Learning}
\label{app:subsec:ope-opl-related}
Off-policy evaluation (OPE) and off-policy learning (OPL) aim to evaluate or optimize a target policy using logged data collected under a different logging (behavior) policy.
In contextual bandits, most studies use importance weighting, reward modeling, and doubly robust estimation, including variants for finite-sample control and large action spaces \citep{dudik2011doubly,wang2017optimal,farajtabar2018more,su2019cab,su2020doubly,kallus2021optimal,metelli2021subgaussian,kiyohara2022doubly,kiyohara2023off,saito2022off,saito2023off,saito2025potec}.
OPL is commonly addressed by reward modeling followed by greedy decision rules, or by gradient-based optimization with off-policy corrections \citep{swaminathan2015batch,swaminathan2015counterfactual,joachims2018deep}.

OPE is also central in reinforcement learning (RL), where the goal is to evaluate a target policy from trajectories generated by a behavior policy.
Many methods develop importance-sampling-based estimators, doubly robust estimators, and high-confidence procedures under Markov Decision Process structure \citep{jiang2016doubly, thomas2016data,liu2018breaking,xie2019towards, kallus2020double,liu2020understanding,kiyohara2024towards}.
Offline reinforcement learning further emphasizes distribution shift and extrapolation issues when the target policy induces regions poorly covered by logged data.

\subsection{Positivity and Overlap Violations}
\label{app:subsec:positivity-violate}
Most point-identification results for policy value rely on a positivity or overlap condition, which requires that any action used by the target policy has nonzero probability under the logging policy at the relevant context, or that the target state-action distribution is supported by the logged data.
When overlap fails, importance-weighted estimators can become unstable and complete separation can induce irreducible bias because the data contain no information about rewards in unsupported regions.
This issue is widely discussed in causal inference, together with diagnostics and practical responses to limited overlap \citep{crump2009dealing,petersen2012diagnosing,d2021overlap}.
Common remedies such as trimming, truncation, or modified weighting can stabilize estimation, but they typically change the target population or avoid making point claims about fully unsupported subpopulations \citep{hill2013assessing,yang2018asymptotic,zhu2021core}.
This is misaligned with applications where the aim is to evaluate or intervene on precisely those unsupported regions, for example hidden positives.

\subsection{Methods Tailored to Overlap Violations}
\label{app:subsec:methods-for-overlap-violate}
One approach is to optimize conservative objectives such as lower confidence bounds or pessimistic value estimates.
These methods provide safety guarantees by avoiding extrapolation into unsupported regions, and are widely used in offline RL through pessimism or behavior regularization \citep{fujimoto2019off, kumar2019stabilizing, kumar2020conservative}.
Related perspectives also appear in contextual bandits and policy learning without overlap, where the objective targets a lower bound rather than the point value \citep{khan2024off, jin2025policy}.
These approaches are valuable when conservative decision making is the goal, but the optimized policy need not maximize the original policy value.

\subsubsection{Side Information and Structured Generalization}
\label{app:subsubsec:side-information}
Another approach exploits additional structure to generalize across actions or contexts.
In large action spaces, methods can leverage action features or embeddings to share statistical strength across actions and reduce the impact of sparse coverage \citep{saito2022off,felicioni2022off,gui2022causal,saito2023off,saito2025potec}.
These methods can be effective when suitable side information is available and modeling assumptions are appropriate, but they require additional inputs or structure beyond standard logged bandit tuples.
Moreover, embedding information may introduce new sources of confounding.

\subsubsection{Changing the Target Intervention to Relax Positivity}
\label{app:subsubsec:changing-target}
A different line of work changes the target intervention to one that remains compatible with the support of the logging distribution.
Incremental propensity score interventions provide a principled way to define such targets in causal inference \citep{kennedy2019nonparametric}.
Recent work uses the idea for positivity-free policy learning, which yields well-defined learning objectives even when standard positivity fails \citep{zhao2024positivity}.
These methods address positivity by redefining the target intervention, which can be desirable in some applications but differs from evaluating the original target policy.

\subsubsection{Partial Identification and Value Bounds}
\label{app:subsubsec:partial-identification}
When overlap fails, partial identification characterizes an identification region for the policy value under weak assumptions.
Classical results provide sharp bounds on treatment effects and show how additional assumptions can tighten identification regions \citep{manski1990nonparametric,balke1997bounds}.
Analogous questions arise in OPE under overlap violations, where one can derive computable value bounds, sometimes under smoothness or other regularity assumptions \citep{khan2024off,jin2025policy}.
This perspective clarifies what can and cannot be learned from data when some regions are unobserved.

\subsection{Positioning of DOLCE}
\label{app:subsec:position-dolce}
DOLCE addresses complete separation in contextual bandits by exploiting temporal structure in logged interactions.
Rather than optimizing conservative lower bounds, redefining the intervention, or requiring action-side features, DOLCE uses routinely available lagged contexts and constructs estimators that target the original policy value.
Its key requirements are a lag-overlap condition and a reward-model property formalized as residual invariance, which generalizes local correctness \citep{saito2022off,saito2025potec}.
This yields a complementary route to handling support violations within a point-estimation framework under explicit modeling conditions.
\section{Supplementary Technical Material}
\label{app:sec:supplementary-technical}

\subsection{Gradient-based OPL baselines}
\label{app:subsec:baseline-gradients}
For completeness, the IPS and DR policy-gradient estimators corresponding to \eqref{eq:ips} and \eqref{eq:dr} are
\begin{align*}
    \widehat{\nabla_{\theta} V}_{\mathrm{IPS}}(\pi_\theta)
    &:=
    \frac{1}{n}\sum_{i=1}^{n}\frac{\pi_{\theta}(A_i\mid X_i)}{\pi_0(A_i\mid X_i)}R_is_{\theta}(A_i\mid X_i)
    ,
    \\
    \widehat{\nabla_{\theta}V}_{\mathrm{DR}}(\pi_{\theta})
    &:=
    \frac{1}{n}\sum_{i=1}^{n}
    \Biggl[
        \frac{\pi_{\theta}(A_i\mid X_i)}{\pi_0(A_i\mid X_i)}\{R_i-\hat{q}(X_i, A_i)\}s_{\theta}(A_i\mid X_i)
        \\
        &\qquad
        + \sum_{a\in \mathcal{A}}\pi_{\theta}(a\mid X_i)\hat{q}(X_i, a)s_{\theta}(a\mid X_i)
    \Biggr]
    .
\end{align*}

\subsection{Impact of Support Violation}
\label{app:subsec:impact-of-support-violation}

\begin{lemma}[Bias under support violation]
\label{lem:bias-support}
    Under Assumptions~\ref{ass:ci1:consistency} and \ref{ass:ci2:unconfoundedness}, the following hold.
    Let $\Delta_{\hat{q}, q}(x,a) := \hat{q}(x,a) - q(x,a)$.
    \begin{enumerate}
        \item The bias of $\hat{V}_{\mathrm{IPS}}(\pi_{\theta})$ satisfies
        \begin{equation*}
            \mathbb{E}_{P}\Bigl[
                \hat{V}_{\mathrm{IPS}}(\pi_{\theta})
            \Bigr] - V(\pi_{\theta})
            =
            -\mathbb{E}_{p(x)}\Biggl[
                \sum_{a\in\mathcal{U}(X;\pi_\theta,\pi_0)}\pi_{\theta}(a\mid X)q(X,a)
            \Biggr]
            .
        \end{equation*}
        \item The bias of $\hat{V}_{\mathrm{DR}}(\pi_{\theta})$ satisfies
        \begin{equation*}
            \mathbb{E}_{P}\Bigl[
                \hat{V}_{\mathrm{DR}}(\pi_{\theta})
            \Bigr] - V(\pi_{\theta})
            =
            \mathbb{E}_{p(x)}\Biggl[
                \sum_{a\in\mathcal{U}(X;\pi_{\theta},\pi_0)}\pi_{\theta}(a\mid X)\Delta_{\hat{q},q}(X,a)
            \Biggr]
            .
        \end{equation*}
        \item The bias of $\widehat{\nabla_{\theta}V}_{\mathrm{IPS}}(\pi_{\theta})$ satisfies
        \begin{equation*}
            \mathbb{E}_{P}\Bigl[
                \widehat{\nabla_{\theta} V}_{\mathrm{IPS}}(\pi_{\theta})
            \Bigr]-\nabla_{\theta} V(\pi_{\theta})
            =
            -\mathbb{E}_{p(x)}\Biggl[
                \sum_{a\in\mathcal{U}(X;\pi_{\theta},\pi_0)}\pi_{\theta}(a\mid X)q(X,a)s_{\theta}(a\mid X)
            \Biggr]
            .
        \end{equation*}
        \item The bias of $\widehat{\nabla_{\theta}V}_{\mathrm{DR}}(\pi_\theta)$ satisfies
        \begin{equation*}
            \mathbb{E}_{P}\Bigl[
                \widehat{\nabla_{\theta}V}_{\mathrm{DR}}(\pi_{\theta})
            \Bigr] - \nabla_{\theta}V(\pi_{\theta})
            =
            \mathbb{E}_{p(x)}\Biggl[
                \sum_{a\in\mathcal{U}(X;\pi_{\theta},\pi_0)}\pi_\theta(a\mid X)\Delta_{\hat{q},q}(X,a)s_{\theta}(a\mid X)
            \Biggr]
            .
        \end{equation*}
    \end{enumerate}
\end{lemma}
See Appendix~\ref{app:proof-bias-support} for the proof of Lemma~\ref{lem:bias-support}.

Lemma~\ref{lem:bias-support} highlights that, under support violation, bias is unavoidable unless either
(i) unsupported actions have negligible mass under $\pi_{\theta}$ or
(ii) the reward model error is controlled on unsupported actions, both of which are challenging in practice.

\subsection{Relating $q$ and $q_k$}
\label{app:subsec:q-qk-relation}

\begin{lemma}[Relating $q$ and $q_k$ under current-action sufficiency]
\label{lem:q-qk-relation}
    Fix a lag $k\in [K]$ and suppose Assumption~\ref{ass:current-action-sufficiency} holds.
    Then, for each $a\in \mathcal{A}$,
    \begin{equation*}
        q(X,a)
        =
        \mathbb{E}_P\bigl[q_k(X,X^{(k)},a)\mid X\bigr]
        \quad
        P\text{-a.s.}
    \end{equation*}
    Consequently,
    \begin{equation*}
        V(\pi_\theta)
        =
        \mathbb{E}_P\Biggl[
            \sum_{a\in\mathcal{A}}\pi_{\theta}(a\mid X)q_k(X,X^{(k)},a)
        \Biggr]
        .
    \end{equation*}
\end{lemma}
See Appendix~\ref{app:proof-q-qk-relation} for the proof of Lemma~\ref{lem:q-qk-relation}.

\subsection{Residual Invariance: Equivalent Characterizations}
\label{app:subsec:residual-invariance}
Assumption~\ref{ass:ri} is equivalent to $\Delta_k(X,X^{(k)},A;\tilde{q}_k)$ being $\sigma(X^{(k)},A)$-measurable $P$-a.s.
Moreover, $\mathrm{ALC}_k(\tilde{q}_k)=0$ if and only if Assumption~\ref{ass:ri} holds.

\subsection{Weight Clipping}
\label{app:subsec:weight-clipping}
In practice, we may use clipped lag weights to stabilize finite-sample behavior:
\begin{equation*}
    w_{k}^{(d)}(x^{(k)},a)
    :=
    \min\{w_k(x^{(k)},a), d\}
    \quad
    d<\infty
    .
\end{equation*}
All unbiasedness statements for oracle weights refer to the unclipped weights $w_k$.
Clipping introduces bias through the standard term $\mathbb{E}[(w_k^{(d)}-w_k)\{R-\tilde{q}_k\}]$, which is small when clipping is rare or when $\bar{\pi}_{0,k}$ is uniformly bounded away from $0$.

\subsection{Oracle identities for lag-$k$ DOLCE}
\label{app:subsec:oracle-dolce-identities}

This subsection collects finite-sample oracle identities for the lag-$k$ DOLCE estimator~\eqref{eq:dolce-ope-k}.
Proofs are deferred to Appendix~\ref{app:sec:proofs}.

\begin{proposition}[Bias of lag-$k$ DOLCE under oracle weights]
\label{prop:bias-dolce-ope}
    Fix $k\in[K]$ and let $\tilde{q}_k$ be any measurable function.
    Consider $\hat{V}_k$ in \eqref{eq:dolce-ope-k} with $\hat{w}_k\equiv w_k$ and $\hat{q}_k \equiv \tilde{q}_k$.
    Then
    \begin{equation*}
        \begin{split}
            \mathbb{E}_P\bigl[\hat{V}_k(\pi_\theta)\bigr] - V(\pi_\theta)
            =
            \mathbb{E}_P\Biggl[
                \sum_{a\in\mathcal{A}}\Big\{
                    \pi_0(a\mid X)w_k(X^{(k)},a)-\pi_{\theta}(a\mid X)
                \Big\}
                \Delta_k(X,X^{(k)},a;\tilde{q}_k)
            \Biggr]
            ,
        \end{split}
    \end{equation*}
    where $\pi_0(a\mid X)$ is the logging policy under $P$.
\end{proposition}
See Appendix~\ref{app:proof-bias-dolce-ope} for the proof.

\begin{proposition}[Variance decomposition for oracle lag-$k$ DOLCE]
\label{prop:var-dolce-ope-oracle}
    Fix $k\in[K]$ and consider the oracle lag-$k$ estimator \eqref{eq:dolce-ope-k} with known weights $w_k$ and a deterministic reward model $\tilde{q}_k$ (not estimated from the same sample).
    Let $O:=(X,X^{(k)},A,R)\sim P$ and assume $\mathbb{E}_P[\psi_k(O)^2]<\infty$, where
    \begin{equation*}
        \psi_k(O)
        :=
        w_k(X^{(k)},A)\{R-\tilde q_k(X,X^{(k)},A)\}
        +
        \sum_{a\in\mathcal{A}}\pi_\theta(a\mid X)\tilde q_k(X,X^{(k)},a).        
    \end{equation*}
    Then
    \begin{equation*}
        \mathbb{V}_P(\hat{V}_k(\pi_\theta))
        =
        \frac{1}{n}\mathbb{V}_P(\psi_k(O))
        .
    \end{equation*}
    Moreover, letting $\sigma_k^2(x,x^{(k)},a):=\mathbb{V}_P(R\mid X=x,X^{(k)}=x^{(k)},A=a)$ and $\Delta_k:=q_k-\tilde q_k$, the law of total variance yields
    \begin{equation*}
        \begin{split}
            \mathbb{V}_P(\psi_k(O))
            &=
            \mathbb{E}_P\Bigl[w_k(X^{(k)},A)^2\sigma_k^2(X,X^{(k)},A)\Bigr]
            \\
            &\quad
            + \mathbb{V}_P\Biggl(
                w_k(X^{(k)},A)\Delta_k(X,X^{(k)},A)
                +
                \sum_{a\in\mathcal{A}}\pi_\theta(a\mid X)\tilde q_k(X,X^{(k)},a)
            \Biggr).
        \end{split}
    \end{equation*}
\end{proposition}
See Appendix~\ref{app:proof-var-dolce-ope-oracle} for the proof.

\subsection{Inference for DOLCE (Cross-Fitting)}
\label{app:subsec:inference}
Throughout this subsection, $\longrightarrow_p$ and $\longrightarrow_d$ denote convergence in probability and convergence in distribution, respectively.
We also use the standard stochastic-order notation $o_p(\cdot)$ and $O_p(\cdot)$.

In practice, $q_k$, $\bar{\pi}_{\theta,k}$, and $\bar{\pi}_{0,k}$ are unknown and must be estimated from data.
We use cross-fitting to avoid strong empirical-process conditions.

\begin{assumption}[Regularity for inference]
\label{ass:regularity}
    $R$ has finite second moment and the (possibly clipped) weights satisfy $\sup_k\|w_k(X^{(k)},A)\|_{L_{\infty}(P)} < \infty$.
\end{assumption}

\begin{assumption}[Nuisance rates]
\label{ass:nuisance-rates}
    Let $\hat{q}_k$ and $\hat{w}_k$ be cross-fitted estimates of $q_k$ and $w_k$, respectively.
    For each fixed $k$,
    \begin{equation*}
        \|\hat{q}_k-q_k\|_{L_2(P)}
        \longrightarrow_p
        0
        ,
        \quad
        \|\hat{w}_k - w_k\|_{L_2(P)}
        \longrightarrow_p
        0
        .
    \end{equation*}
    Moreover, the product rate satisfies
    \begin{equation*}
        \|\hat{q}_k - q_k\|_{L_2(P)}\cdot\|\hat{w}_k - w_k\|_{L_2(P)}
        =
        o_p(n^{-1/2})
        .
    \end{equation*}
    In addition, the reward-model error is asymptotically residual-invariant in the sense that
    \begin{equation*}
        \left\|
            (q_k-\hat{q}_k)
            -
            \mathbb{E}_{P}[q_k-\hat{q}_k\mid X^{(k)},A]
        \right\|_{L_2(P)}
        =
        o_p(n^{-1/2})
        .
    \end{equation*}
\end{assumption}

\begin{theorem}[Consistency and asymptotic normality]
\label{thm:an-dolce-ope}
    Fix $k\in [K]$ and a target policy $\pi_{\theta}$.
    Under Assumptions~\ref{ass:ci1:consistency}--\ref{ass:ci2:unconfoundedness}, \ref{ass:current-action-sufficiency}, \ref{ass:common-lag-support}, \ref{ass:regularity}, and \ref{ass:nuisance-rates},
    the cross-fitted estimator $\hat{V}_k(\pi_\theta)$ is consistent:
    \begin{equation*}
        \hat{V}_k(\pi_{\theta}) \longrightarrow_p V(\pi_\theta)
        .
    \end{equation*}
    Moreover, letting $O:=(X,X^{(k)},A,R)$ and defining
    \begin{equation*}
        \begin{split}
            \phi_k(O)
            &:=
            w_k(X^{(k)},A)\{R-q_k(X,X^{(k)},A)\}
            \\
            &\quad
            +
            \sum_{a\in \mathcal{A}}\pi_{\theta}(a\mid X)q_k(X,X^{(k)},a)
            -
            V(\pi_\theta)
            ,
        \end{split}
    \end{equation*}
    we have the asymptotic normality
    \begin{equation*}
        \sqrt{n}\{\hat{V}_k(\pi_{\theta}) - V(\pi_{\theta})\}
        \longrightarrow_d
        N(0, \mathbb{V}_{P}(\phi_k(O)))
        .
    \end{equation*}
\end{theorem}
See Appendix~\ref{app:proof-an-dolce-ope} for the proof of Theorem~\ref{thm:an-dolce-ope}.

We provide a standard cross-fitted semiparametric proof tailored to the DOLCE moment structure.

\subsection{Lag-Marginal Score for OPL}
\label{app:subsec:dolce-opl-details}

\begin{assumption}[Differentiation under conditional expectation]
\label{ass:diff-under-ce}
    For each $k\in[K]$ and $a\in \mathcal{A}$, the map $\theta\mapsto\pi_\theta(a\mid x)$ is a.s.\ differentiable and there exists an integrable envelope $M(X)$ such that $\sup_{\theta\in\Theta}\|\nabla_{\theta}\pi_{\theta}(a\mid X)\|\le M(X)$ a.s.
    This ensures
    \begin{equation*}
        \begin{split}
            \nabla_{\theta}\bar{\pi}_{\theta,k}(a\mid x^{(k)})
            &=
            \mathbb{E}_P\bigl[
                \nabla_{\theta}\pi_{\theta}(a\mid X)\mid X^{(k)}=x^{(k)}
            \bigr]
            ,
            \\
            &=
            \mathbb{E}_P\bigl[
                \pi_\theta(a\mid X)s_{\theta}(a\mid X)\mid X^{(k)} = x^{(k)}
            \bigr]
            .
        \end{split}
    \end{equation*}
\end{assumption}

Define
\begin{equation*}
    m_{\theta,k}(a\mid x^{(k)})
    :=
    \mathbb{E}_P[\pi_{\theta}(a\mid X)s_{\theta}(a\mid X)\mid X^{(k)}=x^{(k)}]
    .
\end{equation*}
Then $\bar{s}_{\theta,k}(a\mid x^{(k)}) = m_{\theta,k}(a\mid x^{(k)})/\bar{\pi}_{\theta,k}(a\mid x^{(k)})$ whenever $\bar{\pi}_{\theta,k}(a\mid x^{(k)}) > 0$.

\subsection{MTRI targets $\mathrm{ALC}_k$}
\label{app:subsec:mtri-alc}

\begin{proposition}[Interpretation: MTRI targets $\mathrm{ALC}_k$]
\label{prop:mtri-targets-alc}
    Assume $\mathcal{F}_k$ is rich enough to approximate $L_{2,0}(\mathcal{G}_k)$ and the centering step consistently approximates $\mathbb{E}_P[\cdot\mid \mathcal{G}_k]$.
    Then the critic term in \eqref{eq:mtri-pop} lower bounds $\mathrm{ALC}_k(q')$ and becomes equal to $\mathrm{ALC}_k(q')$ in the ideal limit $\tilde{\mathcal{F}}_k \to L_{2,0}(\mathcal{G}_k)$.
    Consequently, minimizing $\mathcal{L}_k(q')$ drives $\mathrm{ALC}_k(q')$ toward zero whenever the class $\mathcal{Q}_k$ contains a residual-invariant approximation (Assumption~\ref{ass:ri}).
\end{proposition}

See Appendix~\ref{app:proof-mtri-targets-alc} for the proof of Proposition~\ref{prop:mtri-targets-alc}.

\subsection{Cross-fitting and confidence intervals}
\label{app:subsec:cross-fitting-ci}
We use $K_{\mathrm{cf}}$-fold cross-fitting throughout.
For inference in OPE, define the cross-fitted influence-function estimate for lag $k$:
\begin{equation*}
    \begin{split}
        \hat{\phi}_{k,i}
        &:=
        \hat{w}_k^{(-\kappa(i))}(X_i^{(k)}, A_i)\{R_i-\hat{q}_k^{(-\kappa(i))}(X_i,X_i^{(k)}, A_i)\}
        \\
        &\qquad
        +
        \sum_{a\in\mathcal{A}}\pi_{\theta}(a\mid X_i)\hat{q}_k^{(-\kappa(i))}(X_i,X_i^{(k)},a) - \hat{V}_k(\pi_\theta)
        ,
    \end{split}
\end{equation*}
where $\kappa(i)$ denotes the fold index.
Then
\begin{equation*}
    \widehat{\mathrm{SE}}_k
    :=
    \sqrt{\frac{1}{n^2}\sum_{i=1}^n \hat{\phi}_{k,i}^2}
    ,
    \quad
    \mathrm{CI}_{1-\alpha}: \hat{V}_k(\pi_\theta)\pm z_{1-\alpha/2}\widehat{\mathrm{SE}}_k
    .
\end{equation*}
\section{Proofs}
\label{app:sec:proofs}

Let $(\Omega, \mathcal{F}, P)$ be an underlying probability space.
For the (single-lag) lagged contextual bandit analysis, define the random tuple $Z:=(X_0,X,A,R)$ taking values in $\mathcal{Z} := \mathcal{X}\times \mathcal{X}\times \mathcal{A}\times \mathbb{R}$, where $X_0$ is the lagged context and $X$ is the current context.
Let $P_0$ denote the logging distribution of $Z$ induced by
\begin{equation*}
    X_0 \sim p_0
    ,
    \;
    X\mid X_0 \sim p(\cdot \mid X_0)
    ,
    \;
    A\mid X\sim \pi_0(\cdot\mid X)
    ,
    \;
    R\mid (X_0,X,A)\sim p(\cdot\mid X_0,X,A)
    ,
\end{equation*}
and define the conditional mean reward $q_0(x,x_0,a) := \mathbb{E}_{P_0}[R\mid X=x,X_0=x_0,A=a]$.
Note that $P_0$ corresponds to the data-generating law $P$ in the main text (restricted to the single-lag setting);
we use $P_0$ here to distinguish the population law from the empirical measure $P_n$ introduced later.
All expectations/variances indexed by a measure are with respect to the corresponding law, e.g., $\mathbb{E}_{P_{0}}[\cdot]$ and $\mathbb{V}_{P_0}[\cdot]$.
For any policy $\pi(\cdot\mid x)$, define the target value
\begin{equation*}
    V(\pi)
    :=
    \mathbb{E}_{X_0\sim p_0}\mathbb{E}_{X\sim p(\cdot\mid X_0)}\bigl[
        \mathbb{E}_{A\sim \pi(\cdot\mid X)}[q_0(X,X_0,A)]
    \bigr]
    .
\end{equation*}
When integrals are needed (continuous actions/contexts), replace summations by integrals with respect to the appropriate dominating measures;
measurability/integrability are assumed whenever invoked.

\subsection{Proof of Lemma~\ref{lem:bias-support}}
\label{app:proof-bias-support}

\begin{proof}[Proof of Lemma~\ref{lem:bias-support}]
    We prove the OPE statements;
    the OPL (gradient) analogs follow by the same algebra after multiplying the integrands by the score.

    Let $(X,A,R)\sim P$ denote the contextual-bandit logging distribution with $A\mid X\sim \pi_0(\cdot\mid X)$, and let $q(x,a):=\mathbb{E}_P[R\mid X=x,A=a]$.
    Fix any target policy $\pi(\cdot\mid x)$ and define the unsupported action set
    \begin{equation*}
        \mathcal{U}(x;\pi,\pi_0)
        :=
        \{a\in\mathcal{A}:\ \pi(a\mid x)>0,\ \pi_0(a\mid x)=0\}
        .
    \end{equation*}

    \paragraph{(i) IPS.}
    Recall
    \begin{equation*}
        \widehat{V}_{\mathrm{IPS}}(\pi)
        :=
        \frac{1}{n}\sum_{i=1}^n \frac{\pi(A_i\mid X_i)}{\pi_0(A_i\mid X_i)}R_i
        .
    \end{equation*}
    Since $A\mid X\sim \pi_0(\cdot\mid X)$, we have
    \begin{equation*}
        \begin{split}
            \mathbb{E}_{P_0}\Bigl[
                \widehat{V}_{\mathrm{IPS}}(\pi)
            \Bigr]
            &=
            \mathbb{E}_{P_0}\Biggl[
                \frac{\pi(A\mid X)}{\pi_0(A\mid X)}R
            \Biggr]
            ,
            \\
            &=
            \mathbb{E}_{X}\Biggl[
                \sum_{a\in\mathcal{A}}\pi_0(a\mid X)\frac{\pi(a\mid X)}{\pi_0(a\mid X)}\mathbb{E}[R\mid X,a]
            \Biggr]
            ,
            \\
            &=
            \mathbb{E}_{X}\Biggl[
                \sum_{a:\pi_0(a\mid X)>0}\pi(a\mid X)q(X,a)
            \Biggr]
            .
        \end{split}
    \end{equation*}
    where the restriction $\pi_0(a\mid X)>0$ is harmless because $P(A=a\mid X)=0$ when $\pi_0(a\mid X)=0$.
    On the other hand $V(\pi) = \mathbb{E}_{X}[\sum_{a\in\mathcal{A}}\pi(a\mid X)q_0(X,a)]$.
    Subtracting yields
    \begin{equation*}
        \begin{split}
            \mathbb{E}_{P_0}\Bigl[\widehat{V}_{\mathrm{IPS}}(\pi)\Bigr] - V(\pi)
            &=
            \mathbb{E}_{X}\Biggl[
                \sum_{a:\pi_0(a\mid X)>0}\pi(a\mid X)q(X,a)
            \Biggr]
            -
            \mathbb{E}_{X}\Biggl[\sum_{a\in\mathcal{A}}\pi(a\mid X)q(X,a)\Biggr]
            ,
            \\
            &=
            -\mathbb{E}_{X}\Biggl[
                \sum_{a\in\mathcal{U}(X;\pi,\pi_0)}\pi(a\mid X)q(X,a)
            \Biggr]
            .
        \end{split}
    \end{equation*}

    \paragraph{(ii) DM and DR.}
    Let $\hat q$ be any (possibly misspecified) reward model and recall
    \begin{equation*}
        \begin{split}
            \widehat{V}_{\mathrm{DM}}(\pi)
            &:=
            \frac{1}{n}\sum_{i=1}^n \mathbb{E}_{A}[\hat{q}(X_i,A)]
            ,
            \\
            \widehat{V}_{\mathrm{DR}}(\pi)
            &:=
            \frac{1}{n}\sum_{i=1}^n
            \Biggl[
                \frac{\pi(A_i\mid X_i)}{\pi_0(A_i\mid X_i)}\{R_i-\hat{q}(X_i,A_i)\}
                +
                \mathbb{E}_{A}\bigl[\hat{q}(X_i,A)\bigr]
            \Biggr]
            .
        \end{split}
    \end{equation*}
    Then
    \begin{equation*}
        \begin{split}
            \mathbb{E}_{P_0}\Bigl[\widehat{V}_{\mathrm{DR}}(\pi)\Bigr]
            &=
            \mathbb{E}_{X}\Biggl[
                \sum_{a\in\mathcal{A}}\pi_0(a\mid X)\frac{\pi(a\mid X)}{\pi_0(a\mid X)}\{q(X,a) - \hat{q}(X,a)\}
                +
                \sum_{a\in\mathcal{A}}\pi(a\mid X)\hat{q}(X,a)
            \Biggr],
            \\
            &=
            \mathbb{E}_{X}\Biggl[
                \sum_{a: \pi_0(a\mid X) > 0}\pi(a\mid X)q(X,a)
                +
                \sum_{a\in \mathcal{U}(X;\pi,\pi_0)}\pi(a\mid X)\hat{q}(X,a)
            \Biggr]
            .
        \end{split}
    \end{equation*}
    Subtracting $V(\pi)=\mathbb{E}_X[\sum_{a\in\mathcal{A}}\pi(a\mid X)q(X,a)]$ yields
    \begin{equation*}
        \mathbb{E}_{P}\Bigl[
            \hat{V}_{\mathrm{DR}}(\pi_{\theta})
        \Bigr] - V(\pi_{\theta})
        =
        \mathbb{E}_{X}\Biggl[
            \sum_{a\in\mathcal{U}(X;\pi,\pi_0)}\pi(a\mid X)\bigl(
                \hat{q}(X,a)-q(X,a)
            \bigr)
        \Biggr]
        ,
    \end{equation*}
    as claimed.

    \paragraph{(iii), (iv) Gradient-based OPL analog.}
    Let $\pi_{\theta}$ be differentiable in $\theta$ and set $s_{\theta}(a\mid x) := \nabla_{\theta}\log \pi_{\theta}(a\mid x)$.
    Repeating the same derivation with the integrand multiplied by $s_{\theta}(a\mid X)$ gives the corresponding bias identities.
\end{proof}

\subsection{Proof of Lemma~\ref{lem:q-qk-relation}}
\label{app:proof-q-qk-relation}

\begin{proof}[Proof of Lemma~\ref{lem:q-qk-relation}]
    Fix $k$ and $a\in\mathcal{A}$.
    By the tower property,
    \begin{equation*}
        \begin{split}
            q(X,a)
            &=
            \mathbb{E}[R\mid X,A=a]
            ,
            \\
            &=
            \mathbb{E}\bigl[\mathbb{E}[R\mid X,X^{(k)},A=a]\mid X,A=a\bigr]
            ,
            \\
            &=
            \mathbb{E}\bigl[q_k(X,X^{(k)},a)\mid X,A=a\bigr].
        \end{split}
    \end{equation*}
    Under Assumption~\ref{ass:current-action-sufficiency}, $A\indep X^{(k)}\mid X$, hence the conditional law of $X^{(k)}$ given $(X,A=a)$ equals that given $X$.
    Therefore,
    \begin{equation*}
        \mathbb{E}\bigl[q_k(X,X^{(k)},a)\mid X,A=a\bigr]
        =
        \mathbb{E}\bigl[q_k(X,X^{(k)},a)\mid X\bigr],
    \end{equation*}
    proving the first claim.
    The second claim follows by multiplying both sides by $\pi_\theta(a\mid X)$, summing over $a$, and taking expectation.
\end{proof}

\subsection{Proof of Theorem~\ref{thm:an-dolce-ope}}
\label{app:proof-an-dolce-ope}

\begin{proof}[Proof of Theorem~\ref{thm:an-dolce-ope}]
    Fix $k\in[K]$ and a target policy $\pi_\theta$.
    Write $O:=(X,X^{(k)},A,R)\sim P$ and let $P_n$ denote the empirical measure based on $\{O_i\}_{i=1}^n$.

    Define, for any candidate nuisances $(q,w)$,
    \begin{equation*}
        \psi_k(O;q,w)
        :=
        w(X^{(k)},A)\{R-q(X,X^{(k)},A)\}
        +
        \sum_{a\in\mathcal{A}}\pi_\theta(a\mid X)q(X,X^{(k)},a)
        .
    \end{equation*}
    Then the (fold-wise) cross-fitted estimator can be written as the empirical mean of $\psi_k(O;\hat q_k,\hat w_k)$ (with fold-specific nuisances); we suppress the fold index to simplify notation.

    By construction, for the oracle nuisances $(q_k,w_k)$ we have
    \begin{equation*}
        \begin{split}
            \mathbb{E}_P\bigl[\psi_k(O;q_k,w_k)\bigr]
            &=
            \mathbb{E}_P\Bigl[w_k(X^{(k)},A)\{R-q_k(X,X^{(k)},A)\}\Bigr]
            \\
            &\qquad
            + \mathbb{E}_P\Biggl[\sum_{a}\pi_\theta(a\mid X)q_k(X,X^{(k)},a)\Biggr]
            \\
            &=
            V(\pi_\theta)
            ,
        \end{split}
    \end{equation*}
    since $\mathbb{E}_P[R-q_k(X,X^{(k)},A)\mid X,X^{(k)},A]=0$ and by Lemma~\ref{lem:q-qk-relation}.
    Therefore $\phi_k(O):=\psi_k(O;q_k,w_k)-V(\pi_\theta)$ is mean-zero and square-integrable under Assumption~\ref{ass:regularity}.

    % \paragraph{Step 1: Asymptotic linear expansion.}
    A standard decomposition gives
    \begin{equation*}
        \hat{V}_k(\pi_\theta)-V(\pi_\theta)
        =
        (P_n-P)\psi_k(\cdot;q_k,w_k)
        +
        R_{1n}
        +
        R_{2n}
        ,
    \end{equation*}
    where
    \begin{equation*}
        R_{1n}
        :=
        P\{\psi_k(\cdot;\hat{q}_k,\hat{w}_k)-\psi_k(\cdot;q_k,w_k)\}
        ,
        \quad
        R_{2n}
        :=
        (P_n-P)\{\psi_k(\cdot;\hat{q}_k,\hat{w}_k)-\psi_k(\cdot;q_k,w_k)\}
        .
    \end{equation*}
    Under cross-fitting and $L_2(P)$-consistency of $(\hat{q}_k,\hat w_k)$, we have $R_{2n}=o_p(n^{-1/2})$ (e.g., by conditional Cauchy--Schwarz within each fold, as in standard cross-fitting arguments).

    Next, expand $R_{1n}$.
    Let $\Delta_q(x,x^{(k)},a):=q_k(x,x^{(k)},a)-\hat{q}_k(x,x^{(k)},a)$ and $\Delta_w(x^{(k)},a):=\hat{w}_k(x^{(k)},a)-w_k(x^{(k)},a)$.
    A direct expansion yields
    \begin{equation*}
        R_{1n}
        =
        P[\Delta_w(X^{(k)},A)\Delta_q(X,X^{(k)},A)]
        +
        B_n
        ,
    \end{equation*}
    where
    \begin{equation*}
        B_n
        :=
        P\Biggl[
        w_k(X^{(k)},A)\Delta_q(X,X^{(k)},A)
        +
        \sum_{a\in\mathcal{A}}\pi_\theta(a\mid X)\{-\Delta_q(X,X^{(k)},a)\}
        \Biggr]
        .
    \end{equation*}
    The first term is controlled by Cauchy--Schwarz:
    \begin{equation*}
        \bigl|P[\Delta_w(X^{(k)},A)\Delta_q(X,X^{(k)},A)]\bigr|
        \le
        \|\hat{w}_k-w_k\|_{L_2(P)}\ \|\hat{q}_k-q_k\|_{L_2(P)}
        =
        o_p(n^{-1/2})
    \end{equation*}
    by Assumption~\ref{ass:nuisance-rates}.

    For $B_n$, note that $B_n$ is linear in the function $\Delta_q$.
    Moreover, if $\Delta_q(X,X^{(k)},A)$ is $\sigma(X^{(k)},A)$-measurable (i.e., Assumption~\ref{ass:ri} holds with $\tilde{q}_k=\hat{q}_k$), then the oracle bias cancellation argument implies $B_n=0$.
    In general, write $\mathcal{G}_k:=\sigma(X^{(k)},A)$ and decompose
    \begin{equation*}
        \Delta_q(X,X^{(k)},A)
        =
        \mathbb{E}_P[\Delta_q(X,X^{(k)},A)\mid \mathcal{G}_k]
        +
        \Bigl(
            \Delta_q(X,X^{(k)},A)-\mathbb{E}_P[\Delta_q(X,X^{(k)},A)\mid \mathcal{G}_k]
        \Bigr).
    \end{equation*}
    The first (projected) component is $\mathcal{G}_k$-measurable and hence does not contribute to $B_n$ by the same cancellation.
    Therefore, $B_n$ depends only on the orthogonal component and satisfies the bound
    \begin{equation*}
        |B_n|
        \le
        C
        \left\|
        \Delta_q(X,X^{(k)},A)-\mathbb{E}_P[\Delta_q(X,X^{(k)},A)\mid X^{(k)},A]
        \right\|_{L_2(P)}
        =
        o_p(n^{-1/2})
        ,
    \end{equation*}
    for a constant $C$ depending only on $\|w_k\|_{L_\infty(P)}$ and $|\mathcal{A}|$ (finite), using Cauchy--Schwarz.
    The last equality uses the final condition in Assumption~\ref{ass:nuisance-rates}.
    Hence $R_{1n}=o_p(n^{-1/2})$ and thus
    \begin{equation*}
        \hat{V}_k(\pi_\theta)-V(\pi_\theta)
        =
        (P_n-P)\phi_k
        +
        o_p(n^{-1/2})
        .
    \end{equation*}

    By the classical Lindeberg--Feller CLT and $\mathbb{E}_P[\phi_k(O)^2]<\infty$,
    \begin{equation*}
        \sqrt{n}(P_n-P)\phi_k
        \longrightarrow_d
        N(0,\mathbb{V}_P(\phi_k(O)))
        .        
    \end{equation*}
    Combining with the expansion above yields the stated asymptotic normality.
    Consistency follows by the same expansion without $\sqrt{n}$-scaling.
\end{proof}

\subsection{Proof of Proposition~\ref{prop:bias-dolce-ope}}
\label{app:proof-bias-dolce-ope}

\begin{proof}[Proof of Proposition~\ref{prop:bias-dolce-ope}]
    Recall the (single-lag) DOLCE estimator (with a fixed lag) takes the form
    \begin{equation*}
    % \label{eq:app-dolce-def}
        \widehat{V}_{\mathrm{DOLCE}}(\pi)
        :=
        \frac{1}{n}\sum_{i=1}^{n}\Bigl[
            w(X_{0i},A_i)\{R_i - \hat{q}(X_i,X_{0i},A_i)\}
            +
            \mu_{\hat{q}}(X_i,X_{0i})
        \Bigr]
        ,
    \end{equation*}
    where
    \begin{equation*}
        \mu_{\hat{q}}(x,x_0)
        :=
        \mathbb{E}_{A\sim \pi(\cdot\mid x)}[\hat{q}(x,x_0,A)]
        ,
    \end{equation*}
    where
    \begin{equation*}
        \bar\pi(a\mid x_0)
        :=
        \mathbb{E}[\pi(a\mid X)\mid X_0=x_0],
        \quad
        \bar\pi_0(a\mid x_0)
        :=
        P(A=a\mid X_0=x_0),
    \end{equation*}
    and
    \begin{equation*}
        w(x_0,a)
        :=
        \frac{\bar\pi(a\mid x_0)}{\bar\pi_0(a\mid x_0)}
        .
    \end{equation*}
    When $\bar{\pi}_0(a\mid x_0) = 0$, interpret $w(x_0,a) = 0$.

    Let $\Delta_{\hat{q}}(x,x_0,a) := q_0(x,x_0,a) - \hat{q}(x,x_0,a)$.
    Using iterated expectation under $P_0$ and $q_0 = \mathbb{E}_{P_0}[R\mid X,X_0,A]$,
    \begin{equation*}
    % \label{eq:app-bias-expand1}
        \begin{split}
            \mathbb{E}_{P_0}[\widehat{V}_{\mathrm{DOLCE}}(\pi)]
            &=
            \mathbb{E}_{P_0}[w(X_0,A)\Delta_{\hat{q}}(X,X_0,A) + \mu_{\hat{q}}(X,X_0)]
            ,
            \\
            &=
            \mathbb{E}_{X_0,X}\Bigl[
                \sum_{a\in\mathcal{A}}\pi_0(a\mid X)w(X_0,a)\Delta_{\hat{q}}(X,X_0,a)
                +
                \sum_{a\in\mathcal{A}}\pi(a\mid X)\hat{q}(X,X_0,a)
            \Bigr]
            .
        \end{split}
    \end{equation*}
    Subtracting $V(\pi) = \mathbb{E}_{X_0,X}[\sum_{a\in\mathcal{A}}\pi(a\mid X)q_0(X,X_0,a)]$ and noting $q_0=\hat{q} + \Delta_{\hat{q}}$ yields the exact bias identity
    \begin{equation}
    \label{eq:app-bias-exact}
        \begin{split}
            \operatorname{Bias}[\widehat{V}_{\mathrm{DOLCE}}(\pi)]
            &=
            \mathbb{E}_{X_0,X}\Bigl[
                \sum_{a\in\mathcal{A}}\Delta_{\hat{q}}(X,X_0,a)\Bigl\{
                    \pi_0(a\mid X)w(X_0,a)-\pi(a\mid X)
                \Bigr\}
            \Bigr]
            .
        \end{split}
    \end{equation}
    This proves the general bias expression claimed in Proposition~\ref{prop:bias-dolce-ope}.

    To obtain the support-violation component explicitly, define $\mathcal{U}_{\mathrm{lag}}(x_0;\pi,\pi_0) := \{a: \bar{\pi}(a\mid x_0)>0, \bar{\pi}_0(a\mid x_0) = 0\}$.
    Because $w(x_0,a) = 0$ on $\{\bar{\pi}_0(a\mid x_0) = 0\}$, \eqref{eq:app-bias-exact} can be rewritten as
    \begin{equation*}
        \begin{split}
            \operatorname{Bias}[\widehat{V}_{\mathrm{DOLCE}}(\pi)]
            &=
            -\mathbb{E}_{X_0,X}\Bigl[
                \sum_{a\in\mathcal{U}_{\mathrm{lag}}(X_0;\pi,\pi_0)}\pi(a\mid X)\Delta_{\hat{q}}(X,X_0,a)
            \Bigr]
            \\
            &\quad
            +
            \mathbb{E}_{X_0,X}\Bigl[
                \sum_{a:\bar{\pi}_0(a\mid X_0)> 0}\Delta_{\hat{q}}(X,X_0,a)\Bigl\{
                    \pi_0(a\mid X)w(X_0,a)-\pi(a\mid X)
                \Bigr\}
            \Bigr]
            ,
        \end{split}
    \end{equation*}
    which isolates the contribution of unsupported lag-actions.

    Finally, under Assumption~\ref{ass:current-action-sufficiency} and \ref{ass:ri} $\Delta_{\hat{q}}(X,X_0, a) = \delta(X_0,a)$ a.s. for some measurable $\delta$, and Assumption~\ref{ass:common-lag-support}, we have $w(X_0,a) = \pi(a\mid X_0)/\pi_0(a\mid X_0)$ well-defined on the support of $\pi(\cdot\mid X_0)$.
    Then, conditioning on $X_0$ and pulling out $\delta(X_0,a)$,
    \begin{equation*}
        \mathbb{E}[\pi_0(a\mid X)w(X_0,a)-\pi(a\mid X)\mid X_0]
        =
        w(X_0,a)\mathbb{E}[\pi_0(a\mid X)\mid X_0] - \mathbb{E}[\pi(a\mid X)\mid X_0]
        .
    \end{equation*}
    By Assumption~\ref{ass:current-action-sufficiency}, $P(A=a\mid X,X_0)= P(A=a\mid X) = \pi_0(a\mid X)$.
    Therefore, by iterated expectations,
    \begin{equation*}
        \begin{split}
            \mathbb{E}[\pi_0(a\mid X)\mid X_0]
            &=
            \mathbb{E}[P(A=a\mid X,X_0)\mid X_0]
            ,
            \\
            &=
            P(A=a\mid X_0)
            ,
            \\
            &=
            \bar{\pi}_0(a\mid X_0)
            .
        \end{split}
    \end{equation*}
    By definition, $\mathbb{E}[\pi(a\mid X)\mid X_0] = \bar{\pi}(a\mid X_0)$.
    Hence $w(X_0,a)\mathbb{E}[\pi_0(a\mid X)\mid X_0] - \mathbb{E}[\pi(a\mid X)\mid X_0] = w(X_0, a)\bar{\pi}_0(a\mid X_0) - \bar{\pi}(a\mid X_0) = 0$.
    Plugging into \eqref{eq:app-bias-exact} shows the bias vanishes.
\end{proof}

\subsection{Proof of Theorem~\ref{thm:unbiased-dolce-ope-oracle}}
\label{app:proof-unbiased-dolce-ope-oracle}

\begin{proof}[Proof of Theorem~\ref{thm:unbiased-dolce-ope-oracle}]
    Under the oracle setting of Theorem~\ref{thm:unbiased-dolce-ope-oracle}, the estimator uses the true lag-marginal policies $\bar{\pi}(\cdot\mid X_0)$ and $\bar{\pi}_0(\cdot\mid X_0)$ and a possibly misspecified $\hat{q}$ that satisfies Assumption~\ref{ass:ri}, i.e., $q_0(X,X_0,a)-\hat{q}(X,X_0,a)=\delta(X_0,a)$ almost surely for some measurable $\delta$.

    By the exact bias identity \eqref{eq:app-bias-exact} in Proposition~\ref{prop:bias-dolce-ope},
    \begin{equation*}
        \operatorname{Bias}[\widehat{V}_{\mathrm{DOLCE}}(\pi)]
        =
        \mathbb{E}_{X_0,X}\Bigl[
            \sum_{a\in\mathcal{A}}\delta(X_0,a)\Bigl\{
                \pi_0(a\mid X)w(X_0,a)
                -
                \pi(a\mid X)
            \Bigr\}
        \Bigr]
        .
    \end{equation*}
    Under Assumption~\ref{ass:common-lag-support}, $w(X_0, a) = \pi(a\mid X_0)/\pi_0(a\mid X_0)$ is well-defined on the support of $\pi(\cdot\mid X_0)$.
    Condition on $X_0$ and apply conditional independence $A\indep X_0\mid X$ as in the previous proof to obtain
    \begin{equation*}
        \mathbb{E}[\pi_0(a\mid X)w(X_0,a)-\pi(a\mid X)\mid X_0]
        =
        w(X_0,a)\bar{\pi}_0(a\mid X_0)-\bar{\pi}(a\mid X_0)
        =
        0
        .
    \end{equation*}
    Therefore the bias is zero, i.e., $\mathbb{E}_{P_0}[\widehat{V}_{\mathrm{DOLCE}}(\pi)] = V(\pi)$.
\end{proof}

\subsection{Proof of Proposition~\ref{prop:var-dolce-ope-oracle}}
\label{app:proof-var-dolce-ope-oracle}

\begin{proof}[Proof of Proposition~\ref{prop:var-dolce-ope-oracle}]
    Since $\hat V_k(\pi_\theta)=n^{-1}\sum_{i=1}^n \psi_k(O_i)$ with i.i.d.\ $O_i\sim P$, we have $\mathbb{V}_P(\hat V_k(\pi_\theta))=n^{-1}\mathbb{V}_P(\psi_k(O))$.
    Conditioning on $(X,X^{(k)},A)$ gives
    \begin{equation*}
        \mathbb{V}_P(\psi_k(O)\mid X,X^{(k)},A)
        =
        w_k(X^{(k)},A)^2\sigma_k^2(X,X^{(k)},A)
    \end{equation*}
    and
    \begin{equation*}
        \mathbb{E}_P[\psi_k(O)\mid X,X^{(k)},A]
        =
        w_k(X^{(k)},A)\Delta_k(X,X^{(k)},A)+\sum_a\pi_\theta(a\mid X)\tilde q_k(X,X^{(k)},a)
        .
    \end{equation*}
    The stated decomposition follows from the law of total variance.
\end{proof}

\subsection{Proof of Proposition~\ref{prop:ri-moment-equiv}}
\label{app:proof-ri-moment-equiv}

\begin{proof}[Proof of Proposition~\ref{prop:ri-moment-equiv}]
    Let $\mathcal{G}_k:=\sigma(X^{(k)},A)$ and view $L_2(P)$ as a Hilbert space with inner product $\langle f,g\rangle:=\mathbb{E}_P[fg]$.
    Let $L_2(\mathcal{G}_k)\subset L_2(P)$ denote the closed subspace of $\mathcal{G}_k$-measurable functions and let
    \begin{equation*}
        L_{2,0}(\mathcal{G}_k)
        :=
        \{f\in L_2(P): \mathbb{E}_P[f\mid \mathcal{G}_k]=0\}        
    \end{equation*}
    be its orthogonal complement.
    
    \paragraph{($\Rightarrow$).}
    If $\Delta$ is $\mathcal{G}_k$-measurable, then for any $f\in L_{2,0}(\mathcal{G}_k)$,
    \begin{equation*}
        \mathbb{E}_P[\Delta f]
        =
        \mathbb{E}_P\Bigl[
            \mathbb{E}_P\bigl[
                \Delta f\mid \mathcal{G}_k
            \bigr]
        \Bigr]
        =
        \mathbb{E}_P\Bigl[
            \Delta\mathbb{E}_P\bigl[
                f\mid \mathcal{G}_k
            \bigr]
        \Bigr]
        =
        0
        .
    \end{equation*}
    
    \paragraph{($\Leftarrow$).}
    Conversely, suppose $\mathbb{E}_P[\Delta f]=0$ for all $f\in L_{2,0}(\mathcal{G}_k)$.
    Take $f^\star:=\Delta-\mathbb{E}_P[\Delta\mid \mathcal{G}_k]$.
    By construction, $f^\star\in L_{2,0}(\mathcal{G}_k)$, hence
    \begin{equation*}
        0
        =
        \mathbb{E}_P[\Delta f^\star]
        =
        \mathbb{E}_P\Bigl[
            \bigl(
                \mathbb{E}_P[\Delta\mid \mathcal{G}_k]+f^\star
            \bigr)f^\star
        \Bigr]
        =
        \mathbb{E}_P\bigl[
            (f^\star)^2
        \bigr]
        ,        
    \end{equation*}
    where we used $\mathbb{E}_P[\mathbb{E}_P[\Delta\mid \mathcal{G}_k]\cdot f^\star]=\mathbb{E}_P[\mathbb{E}_P[\Delta\mid \mathcal{G}_k]\mathbb{E}_P[f^\star\mid \mathcal{G}_k]]=0$.
    Therefore $f^\star=0$ in $L_2(P)$, i.e., $\Delta=\mathbb{E}_P[\Delta\mid \mathcal{G}_k]$ $P$-a.s., which means $\Delta$ is $\mathcal{G}_k$-measurable.
    
    This completes the proof.
\end{proof}

\subsection{Proof of Proposition~\ref{prop:mtri-targets-alc}}
\label{app:proof-mtri-targets-alc}

\begin{proof}[Proof of Proposition~\ref{prop:mtri-targets-alc}]
    Fix $k$ and write $\mathcal{G}_k:=\sigma(X^{(k)},A)$.
    For a candidate $q'\in\mathcal{Q}_k$, define the residual
    \begin{equation*}
        \Delta:=q_k(X,X^{(k)},A)-q'(X,X^{(k)},A)\in L_2(P)
        .
    \end{equation*}
    Let $\Pi_{\mathcal{G}_k}\Delta:=\mathbb{E}_P[\Delta\mid \mathcal{G}_k]$ be the $L_2(P)$-projection onto $L_2(\mathcal{G}_k)$ and set $\Delta_\perp:=\Delta-\Pi_{\mathcal{G}_k}\Delta$.
    Then
    \begin{equation*}
        \mathbb{E}_P[\Delta_\perp\mid \mathcal{G}_k]=0,
        \quad
        \|\Delta_\perp\|_{L_2(P)}^2
        =
        \mathbb{E}_P\!\left[\mathbb{V}(\Delta\mid \mathcal{G}_k)\right]
        =
        \mathrm{ALC}_k(q')
        .
    \end{equation*}
    
    Now consider the centered function class
    \begin{equation*}
        \mathcal{F}_0
        :=
        L_{2,0}(\mathcal{G}_k)=\{f\in L_2(P):\mathbb{E}_P[f\mid \mathcal{G}_k]=0\}
        .        
    \end{equation*}
    For any $f\in\mathcal{F}_0$, using $\mathbb{E}_P[R\mid X,X^{(k)},A]=q_k(X,X^{(k)},A)$, we have
    \begin{equation*}
        \mathbb{E}_P[(R-q'(X,X^{(k)},A))f]
        =
        \mathbb{E}_P[\Delta f]
        =
        \mathbb{E}_P[\Delta_\perp f],
    \end{equation*}
    since $\Pi_{\mathcal{G}_k}\Delta\in L_2(\mathcal{G}_k)$ is orthogonal to $\mathcal{F}_0$.
    Therefore, by Hilbert space duality,
    \begin{equation*}
        \sup_{f\in\mathcal{F}_0:\ \|f\|_{L_2(P)}\le 1}
        \Bigl|\mathbb{E}_P[(R-q'(X,X^{(k)},A))f]\Bigr|
        =
        \|\Delta_\perp\|_{L_2(P)}
        =
        \mathrm{ALC}_k(q')^{1/2}.        
    \end{equation*}
    Squaring yields that the ideal critic term (supremum over $\mathcal{F}_0$) equals $\mathrm{ALC}_k(q')$.
    
    Finally, for a restricted critic class $\tilde{\mathcal{F}}_k$ (and approximate centering), the supremum over $\tilde{\mathcal{F}}_k$ lower bounds the supremum over $\mathcal{F}_0$ up to approximation and centering errors; in the ideal limit $\tilde{\mathcal{F}}_k\to \mathcal{F}_0$ and consistent centering, the lower bound becomes tight.
    This completes the proof.
\end{proof}
\section{Detailed Experiments Settings and Additional Results}
\label{app:sec:detailed-exp}

\subsection{Synthetic Data Experiments}
\label{app:subsec:syn-exp}

\subsubsection{Detailed Data Generating Process}
\label{app:subsubsec:syn-dgp}
This section provides the full synthetic generator used in Section~\ref{subsec:syn-experiments}.
We generate i.i.d.\ samples $(X^{(1)},X,A,R)$ with $X^{(1)},X\in\mathbb{R}^d$ and $A\in\{0,1,\ldots,|\mathcal{A}|-1\}$.

\paragraph{Contexts.}
Let $d$ be the feature dimension.
We draw the lag context as $X^{(1)}\sim N(0,I_d)$ and then draw the current context as
\begin{equation*}
    X \sim N(\rho X^{(1)}, 3^2 I_d),
\end{equation*}
where $\rho\ge 0$ controls lag--current dependence.
To ensure that overlap violations can depend on $X$ without inducing lag-overlap failures, we overwrite the first coordinate $X_1\sim N(0,3^2)$ independently of $X^{(1)}$.

\paragraph{Mean reward.}
The conditional mean reward is
\begin{equation*}
    q(X,X^{(1)},a)
    =
    \lambda g(X,a)
    +
    (1-\lambda)h(X^{(1)},a)
    +
    \eta u(X,X^{(1)},a)
    .
\end{equation*}
The functions $g$ and $h$ are constructed from threshold rules on each feature coordinate.
Concretely, for each feature index $j\in\{0,\ldots,d-2\}$,
\begin{itemize}
\item for the baseline action $a=0$, we add a fixed contrast:
$g(X,0)\text{ and }h(X^{(1)},0)$
receive $-0.2$ if the corresponding feature exceeds $0.5$ and $+0.2$ otherwise;
\item for each $a\ge 1$, we draw environment-specific coefficients (fixed across Monte Carlo replications) and add
a piecewise-constant contribution depending on whether the feature exceeds $0.5$.
\end{itemize}
We also include a non-linear count effect:
let $C(X)=\sum_{j=1}^{d-2}\boldsymbol{1}\{X_{j+1}>0.5\}$ and define an additional term that penalizes action $0$ and rewards $a\ge 1$ when $C(X)\ge 2$;
the same construction applies to $X^{(1)}$ in $h$.
This yields a reward surface that is non-linear but structured.

The interaction term $u$ is used to intentionally violate residual invariance:
it depends on products (and a sinusoidal term) of current and lag features and is action-dependent via random coefficients.
When $\eta=0$, rewards are additive in $(X,X^{(1)})$ and align with the residual-invariance-friendly structure targeted by our estimator;
increasing $\eta$ strengthens non-additive effects.

Finally, observed rewards include additive Gaussian noise:
\begin{equation*}
    R = q(X,X^{(1)},A) + \varepsilon
    ,
    \quad
    \varepsilon\sim N(0,1).
\end{equation*}

\paragraph{Logging policy and support violation ratio.}
The logging policy is a softmax over current-context scores:
\begin{equation*}
    \pi_0(a\mid X)
    \propto
    \exp\{\beta g(X,a)\}
    ,
\end{equation*}
with an optional exploration floor (used in OPL) implemented as a convex mixture with the uniform policy.
To impose current-context support violations, we choose a threshold $c_r$ as the $(1-r)$-quantile of $\{X_{i1}\}_{i=1}^n$ and enforce
\begin{equation*}
    X_1>c_r
    \Rightarrow
    \pi_0(0\mid X)=1,\ \pi_0(a\mid X)=0\ (a\ne 0)
    ,
\end{equation*}
and set $A=0$ deterministically on this region.
We refer to $r\in[0,1]$ as the support violation ratio.

\paragraph{Target policy for OPE.}
In the OPE experiments, the target policy is an $\varepsilon$-greedy policy using \emph{only} the current-context component $g(X,\cdot)$,
i.e., it chooses the empirically best action with probability $1-\varepsilon$ and explores uniformly with probability $\varepsilon$ (default $\varepsilon=0.1$).
The true value $V(\pi_\theta)$ is approximated via a large Monte Carlo sample (independent of $r$ because $r$ affects only the logging policy).

\subsubsection{Evaluation Metrics}
\label{app:subsubsec:syn-metrics}

\paragraph{OPE metrics.}
For each configuration, we generate $B=100$ logged datasets and compute, for each estimator,
\begin{equation*}
    \begin{split}
        \widehat{\mathrm{Bias}}
        &:=
        \frac{1}{B}\sum_{b=1}^B \left(\widehat{V}^{(b)}-V\right)
        ,
        \\
        \widehat{\mathrm{Var}}
        &:=
        \frac{1}{B}\sum_{b=1}^B \left(\widehat{V}^{(b)}-\overline{V}\right)^2
        ,
        \quad
        \overline{V}:=\frac{1}{B}\sum_{b=1}^B \widehat{V}^{(b)}
        ,
        \\
        \widehat{\mathrm{MSE}}
        &:=
        \frac{1}{B}\sum_{b=1}^B \left(\widehat{V}^{(b)}-V\right)^2
        .
    \end{split}
\end{equation*}
To compute coverage, we form a nominal $95\%$ CI for each run using influence-function standard errors.
When an estimator admits a cross-fitted influence-function representation $\widehat{V}=n^{-1}\sum_{i=1}^n \hat{\phi}_i + o_p(n^{-1/2})$, we estimate the asymptotic variance by
\begin{equation*}
    \widehat{\sigma}^2 := \frac{1}{n}\sum_{i=1}^n \hat{\phi}_i^2
    ,
    \quad
    \widehat{\mathrm{SE}} := \frac{\widehat{\sigma}}{\sqrt{n}}
    ,
\end{equation*}
and construct $\widehat{V}\pm z_{0.975}\widehat{\mathrm{SE}}$.
Coverage is the fraction of runs where the CI contains $V$.

\paragraph{OPL metrics.}
We generate an independent test set of size $n_{\mathrm{test}}$ (default $10,000$) and evaluate true policy values by
\begin{equation*}
    \widehat{V}_{\mathrm{test}}(\pi)
    :=
    \frac{1}{n_{\mathrm{test}}}\sum_{j=1}^{n_{\mathrm{test}}}
    \sum_{a\in\mathcal{A}} \pi(a\mid X_j)\, q(X_j,X_j^{(1)},a)
    ,
\end{equation*}
where $q$ is known from the generator.
We compute $V^\star$ on the same test set as $\widehat{V}^\star_{\mathrm{test}} := n_{\mathrm{test}}^{-1}\sum_j \max_a q(X_j,X_j^{(1)},a)$ and measure NI, OSI, and regret as in Section~\ref{subsec:syn-experiments}.

\subsubsection{Results and Additional Sensitivity Analyses}
\label{app:subsubsec:syn-results}

\paragraph{Sensitivity for OPE.}
\label{app:subsubsec:syn-ope-sensitivity}
We study robustness of OPE estimators to key data-generating and statistical factors while fixing the remaining parameters to the default setting (unless otherwise stated).
In particular, we vary
(i) the mixture parameter $\lambda$ controlling the relative strength of current- versus lag-driven reward components,
(ii) the number of actions $|\mathcal{A}|$,
(iii) the logged sample size $n$, and
(iv) the interaction strength $\eta$ that controls violations of residual invariance.
Across all sweeps, we report MSE, bias, variance, and empirical coverage of nominal $95\%$ confidence intervals.

Figure~\ref{fig:app:ope-lambda} shows that DOLCE exhibits a slight deterioration in MSE as $\lambda$ increases.
The bias and variance plots indicate that this change is primarily driven by the emergence of a small bias at larger $\lambda$, which can make DOLCE marginally worse than DM/DR in MSE in this regime, while DOLCE still maintains high (often near-nominal) CI coverage.
This trend is consistent with our theory:
DOLCE's oracle unbiasedness hinges on (approximate) residual invariance, and as the reward becomes more current-context-driven (larger $\lambda$), any residual-invariance defect in the learned reward model can translate into non-negligible bias through the bias identity in Proposition~\ref{prop:bias-dolce-ope}.
We also observe a noticeable regime change around $\lambda=0.1$, which is plausibly explained by the DGP:
overlap violations are induced via the current feature $X_1$, so when $\lambda$ is very small the target value is dominated by lag-driven structure and is less sensitive to the unsupported current-context region, whereas once $\lambda$ becomes non-negligible the current-driven component couples the reward to the violation region more strongly.

\begin{figure*}[tb]
\vskip 0.2in
\begin{center}
\centerline{\includegraphics[width=\columnwidth]{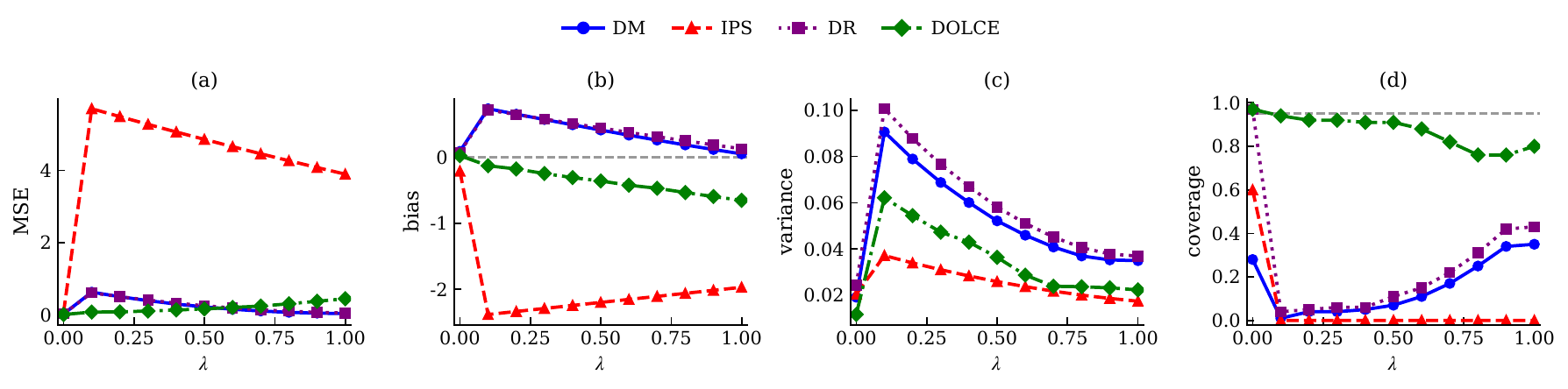}}
\caption{
Off-policy evaluation under support violation: sensitivity to the mixture parameter $\lambda$.
(a) Mean squared error (MSE) of the estimated policy value $\hat{V}(\pi)$ as a function of $\lambda$.
(b) Bias of $\hat{V}(\pi)$ versus $\lambda$, where the horizontal gray dashed line corresponds to zero bias.
(c) Variance of $\hat{V}(\pi)$ versus $\lambda$.
(d) Empirical coverage of nominal $95\%$ confidence intervals for $V(\pi)$ as a function of $\lambda$, where the horizontal gray dashed line denotes the nominal $95\%$ coverage level.
}
\label{fig:app:ope-lambda}
\end{center}
\vskip -0.2in
\end{figure*}

Figure~\ref{fig:app:ope-action} indicates that DOLCE remains stable across different action-set sizes, whereas DM degrades as $|\mathcal{A}|$ increases, suggesting increased model misspecification impact in larger action spaces.
The variance increases with $|\mathcal{A}|$ (Figure~\ref{fig:app:ope-action} (c)), which matches the usual intuition that propensities become smaller on average and importance-weighted components become noisier as the action space grows.
In addition, while DM/IPS/DR exhibit low or unstable coverage, DOLCE consistently maintains coverage close to the nominal level across $|\mathcal{A}|$ (Figure~\ref{fig:app:ope-action} (d)), supporting the practical value of the lag-based correction for inference.

\begin{figure*}[tb]
\vskip 0.2in
\begin{center}
\centerline{\includegraphics[width=\columnwidth]{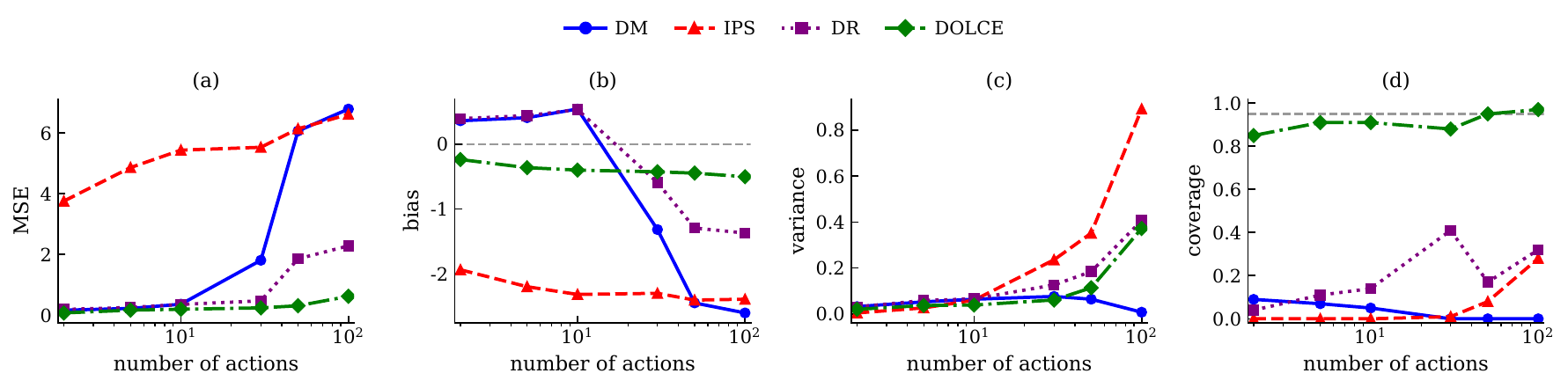}}
\caption{
Off-policy evaluation under support violation: sensitivity to the number of actions $|\mathcal{A}|$.
(a) Mean squared error (MSE) of the estimated policy value $\hat{V}(\pi)$ as a function of $|\mathcal{A}|$.
(b) Bias of $\hat{V}(\pi)$ versus $|\mathcal{A}|$, where the horizontal gray dashed line corresponds to zero bias.
(c) Variance of $\hat{V}(\pi)$ versus $|\mathcal{A}|$.
(d) Empirical coverage of nominal $95\%$ confidence intervals for $V(\pi)$ as a function of $|\mathcal{A}|$, where the horizontal gray dashed line denotes the nominal $95\%$ coverage level.
}
\label{fig:app:ope-action}
\end{center}
\vskip -0.2in
\end{figure*}

Figure~\ref{fig:app:ope-data} shows that DM, DR, and DOLCE are broadly stable across the examined data sizes, with a mild tendency for bias to move closer to zero as $n$ increases.
The variance decreases with $n$ (Figure~\ref{fig:app:ope-data} (c)), as expected from standard $1/n$ scaling and improved nuisance estimation with more data, and the coverage pattern in Figure~\ref{fig:app:ope-data} (d) suggests that this variance reduction can materially affect CI calibration in finite samples.
This behavior is in line with the asymptotic normality result (Theorem~\ref{thm:an-dolce-ope}), where larger $n$ reduces both sampling variability and the impact of nuisance-estimation error under cross-fitting.

\begin{figure*}[tb]
\vskip 0.2in
\begin{center}
\centerline{\includegraphics[width=\columnwidth]{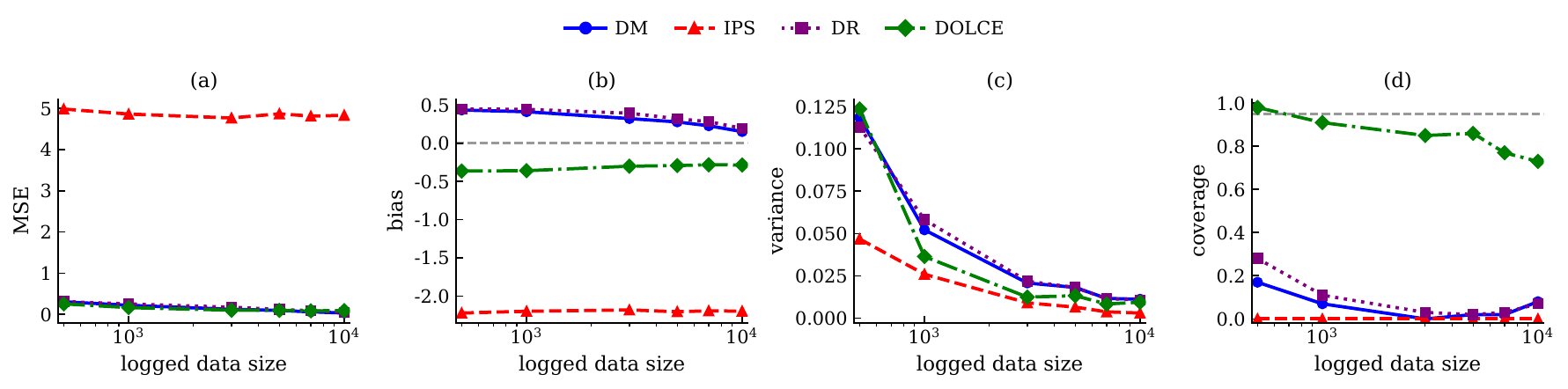}}
\caption{
Off-policy evaluation under support violation: sensitivity to the logged sample size $n$.
(a) Mean squared error (MSE) of the estimated policy value $\hat{V}(\pi)$ as a function of $n$.
(b) Bias of $\hat{V}(\pi)$ versus $n$, where the horizontal gray dashed line corresponds to zero bias.
(c) Variance of $\hat{V}(\pi)$ versus $n$.
(d) Empirical coverage of nominal $95\%$ confidence intervals for $V(\pi)$ as a function of $n$, where the horizontal gray dashed line denotes the nominal $95\%$ coverage level.
}
\label{fig:app:ope-data}
\end{center}
\vskip -0.2in
\end{figure*}

Figure~\ref{fig:app:ope-eta} shows that all methods exhibit no substantial qualitative change across the tested $\eta$ values.
Notably, DOLCE's bias stays around a nearly constant level (approximately $-0.4$ in our runs) rather than worsening with $\eta$.
In principle, increasing $\eta$ strengthens the non-additive interaction term and should make residual invariance harder to satisfy, so Proposition~\ref{prop:bias-dolce-ope} suggests that bias could increase with $\eta$;
the observed flat bias curve therefore indicates that (in this configuration) the interaction term has limited average leverage under the target policy and/or is largely absorbed by the MTRI projection, with the remaining bias dominated by other finite-sample effects such as weight and nuisance estimation.

\begin{figure*}[tb]
\vskip 0.2in
\begin{center}
\centerline{\includegraphics[width=\columnwidth]{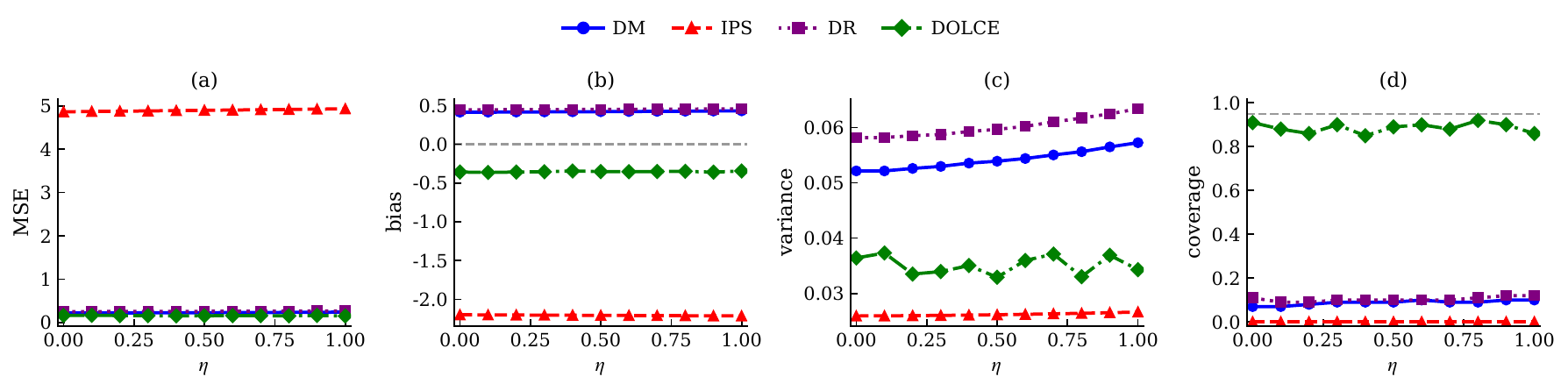}}
\caption{
Off-policy evaluation under support violation: sensitivity to the interaction strength $\eta$.
(a) Mean squared error (MSE) of the estimated policy value $\hat{V}(\pi)$ as a function of $\eta$.
(b) Bias of $\hat{V}(\pi)$ versus $\eta$, where the horizontal gray dashed line corresponds to zero bias.
(c) Variance of $\hat{V}(\pi)$ versus $\eta$.
(d) Empirical coverage of nominal $95\%$ confidence intervals for $V(\pi)$ as a function of $\eta$, where the horizontal gray dashed line denotes the nominal $95\%$ coverage level.
}
\label{fig:app:ope-eta}
\end{center}
\vskip -0.2in
\end{figure*}

\paragraph{Sensitivity for OPL.}
\label{app:subsubsec:syn-opl-sensitivity}
We next examine sensitivity of gradient-based OPL to the same factors (excluding $n$ in our current sweep), reporting normalized improvement (NI), one-step improvement (OSI), and regret.
Unless otherwise stated, other parameters are fixed to the default setting.

Figure~\ref{fig:app:opl-lambda} shows that DR slightly outperforms DOLCE in terms of normalized improvement and regret, whereas DOLCE achieves a slightly larger one-step improvement.
This suggests that DOLCE provides a marginally better update direction at initialization (OSI), even if the final learned policy quality (NI/regret) can be slightly better for DR under these settings.
This pattern is consistent with the interpretation of OSI as a first-order diagnostic of gradient alignment: when lag-based bias cancellation improves the gradient direction, OSI can increase even if multi-step optimization performance is influenced by additional factors such as gradient variance, step-size sensitivity, and accumulated estimation error.

\begin{figure*}[tb]
\vskip 0.2in
\begin{center}
\centerline{\includegraphics[width=\columnwidth]{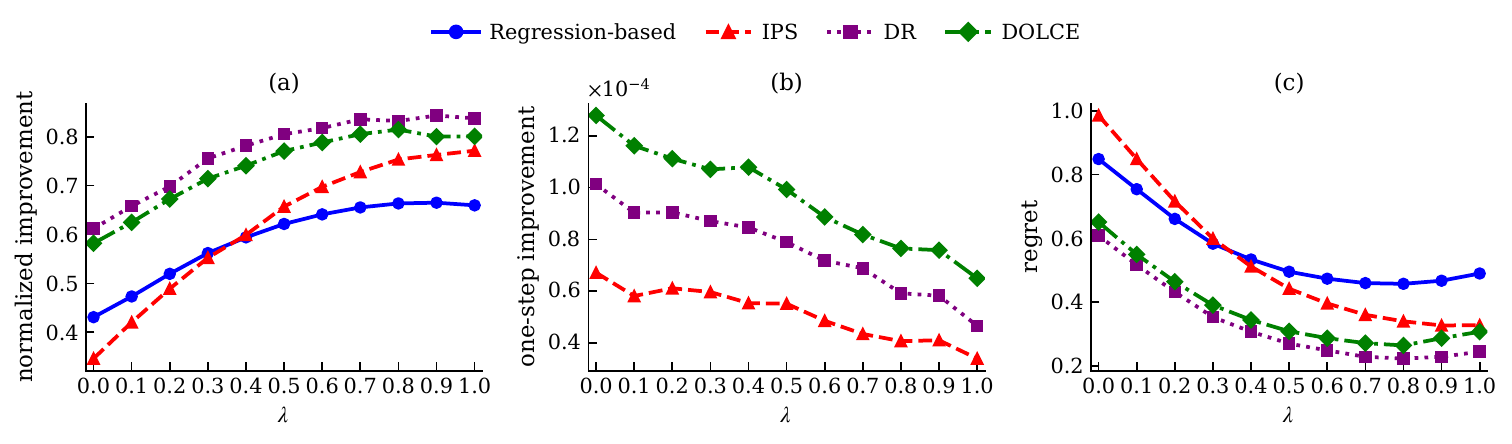}}
\caption{
Off-policy learning under support violation: sensitivity to the mixture parameter $\lambda$.
(a) Normalized improvement over the logging policy as a function of $\lambda$.
(b) One-step improvement versus $\lambda$, reflecting the alignment between the policy update direction and the true value gradient.
(c) Regret as a function of $\lambda$.
}
\label{fig:app:opl-lambda}
\end{center}
\vskip -0.2in
\end{figure*}

Figure~\ref{fig:app:opl-action} shows a broadly similar relationship among methods as in the $\lambda$ sweep, while the overall magnitude of improvement decreases as $|\mathcal{A}|$ grows.
This indicates a limitation induced by larger action spaces, where both value estimation and gradient estimation become harder and the learned policy gains diminish.
This degradation is expected: increasing $|\mathcal{A}|$ typically increases estimation variance (and the optimization difficulty) because probability mass is spread more thinly across actions and informative samples for each action become scarcer.

\begin{figure*}[tb]
\vskip 0.2in
\begin{center}
\centerline{\includegraphics[width=\columnwidth]{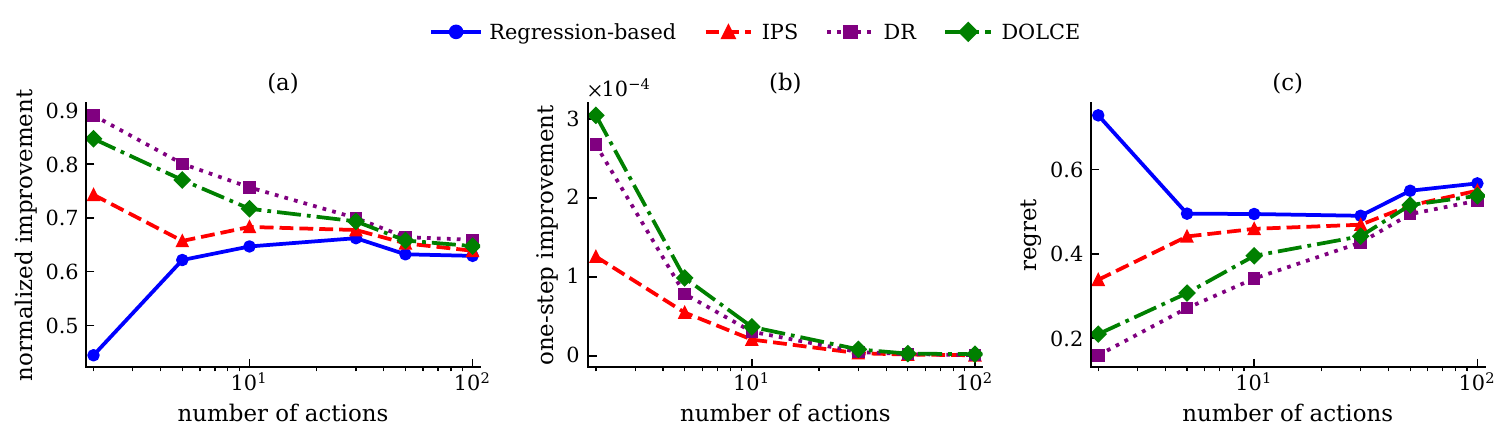}}
\caption{
Off-policy learning under support violation: sensitivity to the number of actions $|\mathcal{A}|$.
(a) Normalized improvement over the logging policy as a function of $|\mathcal{A}|$.
(b) One-step improvement versus $|\mathcal{A}|$, reflecting the alignment between the policy update direction and the true value gradient.
(c) Regret as a function of $|\mathcal{A}|$.
}
\label{fig:app:opl-action}
\end{center}
\vskip -0.2in
\end{figure*}

Figure~\ref{fig:app:opl-data} varies the logged data size $n$.
We observe qualitatively similar patterns to Figures~\ref{fig:app:opl-action} and \ref{fig:app:opl-eta}:
the relative ordering among methods remains largely unchanged, with DR being slightly favorable in terms of normalized improvement and regret in this setting, while DOLCE often yields comparable performance and a competitive (sometimes slightly better) one-step improvement, suggesting a reasonably well-aligned gradient direction under support violation.
As expected, increasing $n$ primarily reduces estimation noise, leading to modestly more stable OPL outcomes across all methods.

\begin{figure*}[tb]
\vskip 0.2in
\begin{center}
\centerline{\includegraphics[width=\columnwidth]{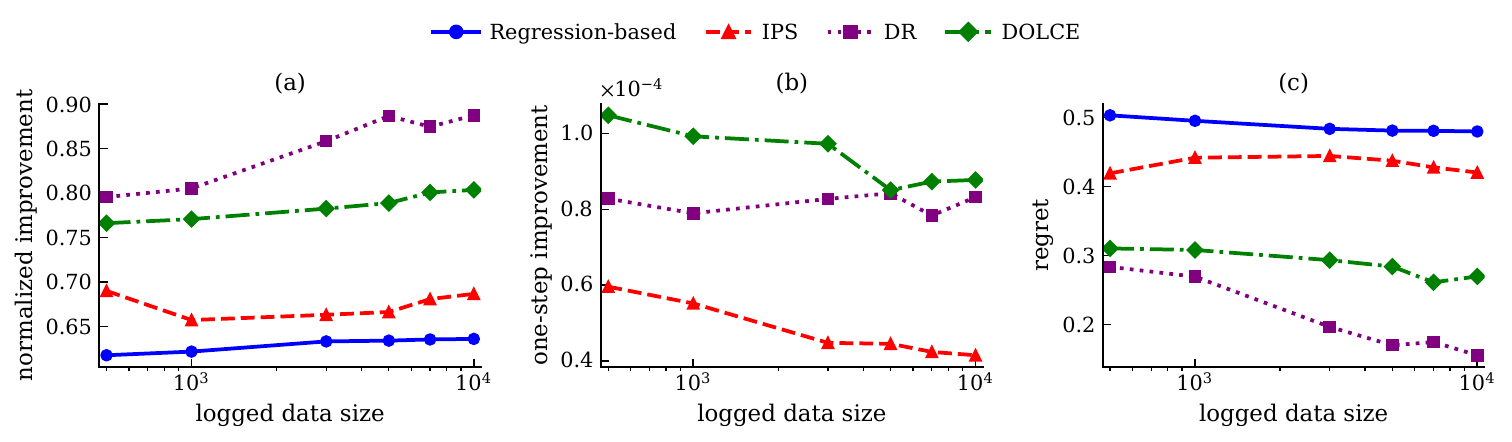}}
\caption{
Off-policy learning under support violation: sensitivity to logged data size.
(a) Normalized improvement over the logging policy as a function of the logged data size $n$.
(b) One-step improvement versus $n$, reflecting the alignment between the policy update direction and the true value gradient.
(c) Regret as $n$ increases.
}
\label{fig:app:opl-data}
\end{center}
\vskip -0.2in
\end{figure*}

Figure~\ref{fig:app:opl-eta} shows that the relative ranking among methods is qualitatively similar to the $\lambda$ sweep, while increasing $\eta$ leads to a mild deterioration in both normalized improvement and regret.
This is consistent with the role of $\eta$ in strengthening current--lag interactions that violate residual invariance, which can reduce the effectiveness of lag-based bias cancellation in the gradient estimate and hence slightly degrade OPL performance.

\begin{figure*}[tb]
\vskip 0.2in
\begin{center}
\centerline{\includegraphics[width=\columnwidth]{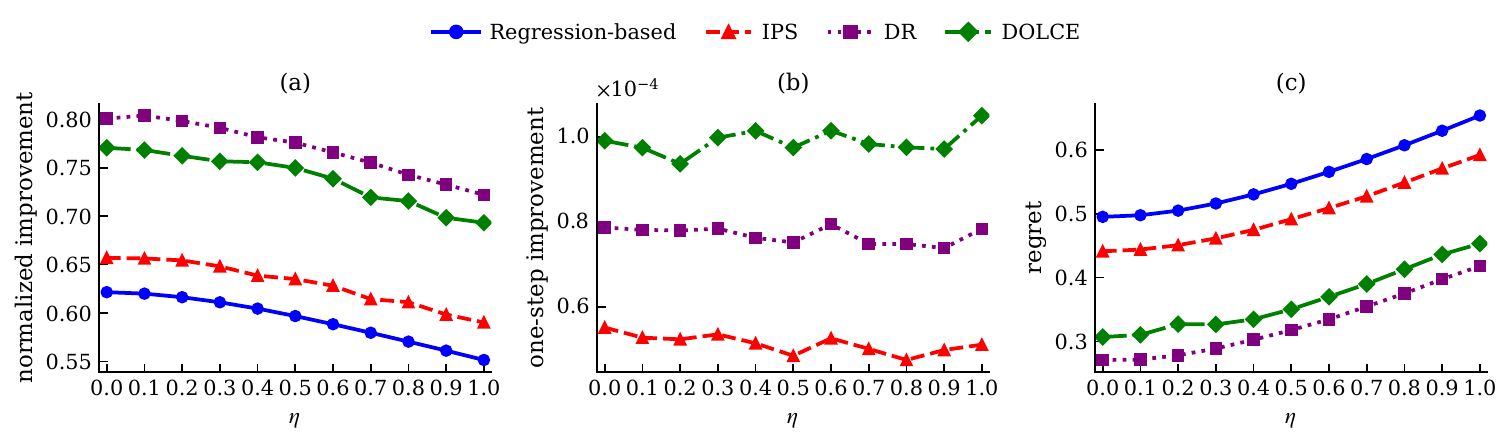}}
\caption{
Off-policy learning under support violation: sensitivity to the interaction strength $\eta$.
(a) Normalized improvement over the logging policy as a function of $\eta$.
(b) One-step improvement versus $\eta$, reflecting the alignment between the policy update direction and the true value gradient.
(c) Regret as a function of $\eta$.
}
\label{fig:app:opl-eta}
\end{center}
\vskip -0.2in
\end{figure*}

\subsection{Real-World Data Experiments}
\label{app:subsec:realworld-nwicu}

This subsection details the NWICU preprocessing pipeline used to construct the lagged contextual bandit dataset for Section~\ref{subsec:real-experiments}.

\paragraph{Data sources.}
We use three NWICU tables:
(i) ICU stay boundaries (icustays) containing \texttt{stay\_id}, \texttt{intime}, and \texttt{outtime};
(ii) bedside charting events (chartevents) containing time-stamped vitals; and
(iii) medication administration records (eMAR) containing time-stamped medication names and admission identifiers.
Because eMAR does not always include \texttt{stay\_id}, we first map eMAR events to ICU stays by joining on \texttt{hadm\_id} and retaining medication times within each ICU stay window.

\paragraph{Vital sign extraction and cleaning.}
We identify the item IDs for four vital signs from the dictionary table (\texttt{d\_items}) restricted to \texttt{linksto=chartevents}:
HR (label \texttt{PULSE}), SpO$_2$ (label \texttt{PULSE OXIMETRY}), SBP (label \texttt{BP SYSTOLIC}), and DBP (label \texttt{BP DIASTOLIC}).
We then read \texttt{chartevents} in large chunks, keep only rows whose \texttt{itemid} matches the selected vital sign IDs, parse timestamps, and coerce numeric values.
We apply physiologically plausible range filters (e.g., HR in [20,250], SpO$_2$ in [50,100], SBP in [40,300], DBP in [20,200]) and drop rows with missing \texttt{stay\_id}, timestamp, or value.
Finally, we keep only charting events that occur within \texttt{intime} and \texttt{outtime} of the corresponding ICU stay.

\paragraph{Hourly decision grid and lagged contexts.}
For each ICU stay, we construct an hourly decision grid with step size $\Delta t=1$ hour.
To ensure that lag features and next-hour reward are well-defined, we require a minimum history window and a one-hour reward horizon:
decision times begin after the minimum history offset and end one hour before ICU discharge.
At each decision time $t$, we compute the current vital sign values using a last-observation-carried-forward rule:
for each variable (HR, SpO$_2$, SBP, DBP), we take the most recent measurement at or before $t$ within a 2-hour tolerance.
We then compute $\mathrm{MAP}_t=(\mathrm{SBP}_t + 2\mathrm{DBP}_t)/3$.

We construct lag contexts at $\{1\mathrm{h},2\mathrm{h},4\mathrm{h}\}$ by shifting the hourly grid within each stay.
We drop any decision times that do not admit all required lagged covariates.

\paragraph{Action definition (vasopressor administration).}
We define a binary action $A_t\in\{0,1\}$ indicating whether any vasopressor medication is administered within the decision window $[t,t+1\mathrm{h})$.
We identify vasopressor administrations from eMAR medication names via a case-insensitive regular expression that matches common vasopressor strings (e.g., norepinephrine, epinephrine, vasopressin, phenylephrine, dopamine).
We assign each medication event to the most recent decision time $t$ within the same stay and set $A_t=1$ if at least one vasopressor event falls in $[t,t+1\mathrm{h})$; otherwise $A_t=0$.

\paragraph{Reward definition (next-hour MAP threshold).}
We define the reward as a one-hour-ahead hemodynamic stability proxy:
$R_t=\mathbf{1}\{\mathrm{MAP}_{t+1\mathrm{h}}\ge 65\}$.
Operationally, $\mathrm{MAP}_{t+1\mathrm{h}}$ is computed on the hourly grid using the same carry-forward procedure, and we drop terminal decision times that do not have a next-hour MAP.

\paragraph{Subsampling to approximate i.i.d.\ bandit draws.}
Because multiple decision points per ICU stay induce within-stay dependence, we subsample decision points to better match the i.i.d.\ assumption used in our theory.
Concretely, we optionally cap the maximum number of decision points per stay (by selecting a contiguous block of fixed length), and then sample a single decision point per stay uniformly at random among the retained points.
This results in one contextual bandit tuple per ICU stay.

\paragraph{Train/evaluation split and nuisance estimation.}
We split the resulting dataset into train/evaluation sets at the ICU-stay level to avoid leakage across the same stay.
We estimate the logging propensity $\pi_0(A\mid X)$ using a regularized logistic regression model on standardized current features.
For DM/DR, we estimate reward models $\hat{q}(x,a)$ using an action-wise multilayer perceptron with cross-fitting.
For DOLCE, we construct lag-specific reward models $\hat{q}_k(x,x^{(k)},a)$ using residual-invariance-friendly additive modeling and compute ALC scores to softmin-weight multiple candidate lags, as described in Section~\ref{sec:proposed-method}.

\subsubsection{OPL: Full Cross-Evaluation Table}
\label{app:subsec:realworld-nwicu:opl-full}

Table~\ref{tab:nwicu-opl-full} reports the value of each learned policy evaluated by each OPE estimator on the held-out split.
We include this table primarily as a diagnostic: large discrepancies across estimators (and out-of-range IPS estimates for bounded rewards) indicate sensitivity to limited overlap and estimator instability.

\begin{table}[t]
\centering
\caption{
Full cross-evaluation of learned policies on NWICU.
Each row evaluates a learned policy using an OPE estimator.
}
\label{tab:nwicu-opl-full}
\begin{tabular}{llrrrr}
\toprule
Learned policy & OPE estimator & $\widehat{V}$ & SE & CI$_{\mathrm{low}}$ & CI$_{\mathrm{high}}$ \\
\midrule
$\hat{\pi}_{\mathrm{DM}}$    & DM    & 0.942 & 0.002 & 0.938 & 0.946 \\
$\hat{\pi}_{\mathrm{DM}}$    & IPS   & 0.996 & 0.076 & 0.846 & 1.146 \\
$\hat{\pi}_{\mathrm{DM}}$    & DR    & 0.895 & 0.034 & 0.828 & 0.961 \\
$\hat{\pi}_{\mathrm{DM}}$    & DOLCE & 0.944 & 0.019 & 0.906 & 0.981 \\
\midrule
$\hat{\pi}_{\mathrm{IPS}}$   & DM    & 0.963 & 0.002 & 0.959 & 0.967 \\
$\hat{\pi}_{\mathrm{IPS}}$   & IPS   & 1.074 & 0.095 & 0.888 & 1.260 \\
$\hat{\pi}_{\mathrm{IPS}}$   & DR    & 0.888 & 0.047 & 0.795 & 0.981 \\
$\hat{\pi}_{\mathrm{IPS}}$   & DOLCE & 0.946 & 0.010 & 0.926 & 0.967 \\
\midrule
$\hat{\pi}_{\mathrm{DR}}$    & DM    & 0.955 & 0.002 & 0.951 & 0.958 \\
$\hat{\pi}_{\mathrm{DR}}$    & IPS   & 0.989 & 0.035 & 0.921 & 1.056 \\
$\hat{\pi}_{\mathrm{DR}}$    & DR    & 0.934 & 0.018 & 0.898 & 0.971 \\
$\hat{\pi}_{\mathrm{DR}}$    & DOLCE & 0.953 & 0.004 & 0.945 & 0.962 \\
\midrule
$\hat{\pi}_{\mathrm{DOLCE}}$ & DM    & 0.954 & 0.002 & 0.951 & 0.958 \\
$\hat{\pi}_{\mathrm{DOLCE}}$ & IPS   & 1.075 & 0.102 & 0.875 & 1.276 \\
$\hat{\pi}_{\mathrm{DOLCE}}$ & DR    & 0.886 & 0.052 & 0.784 & 0.989 \\
$\hat{\pi}_{\mathrm{DOLCE}}$ & DOLCE & 0.949 & 0.006 & 0.937 & 0.962 \\
\bottomrule
\end{tabular}
\end{table}

\end{document}